\documentclass{aicom2enocopyright}

\usepackage{amsthm,amssymb}
\usepackage{url}

\newtheorem{definition}{Definition}
\newtheorem{lemma}{Lemma}
\newtheorem{theorem}{Theorem}
\newtheorem{corollary}{Corollary}  
\newtheorem{remark}{Remark}
\newtheorem{notation}{Notational Remark}
\newtheorem{example}{Example}

\newcommand{\jlc}[1]{}
\newcommand{\pn}[1]{} 
\newcommand{\ipg}[1]{} 

\newcommand{\nonempty}[1]{\mathrm{nonempty}(#1)}

\newcommand{\tuple}[2]{{#1}\!\!-\!\!{\mathit{tuple}}_{{#2}}}

\newcommand{\comment}[1]{}

\newenvironment{program*}{\tt\obeyspaces\begin{bogustabbing}\vspace{-.125in}}{\vspace{-.125in}\end{bogustabbing}}
\newenvironment{program**}{\it\obeyspaces\begin{bogustabbing}\vspace{-.125in}}{\vspace{-.125in}\end{bogustabbing}}

\newcommand{\abit}[1]{\mbox{\hspace{#1em}}}
\newcommand{\support}[1]{{\mathit{support}_{#1}}}
\newcommand{\Support}[3]{{\mathit{support}_{\pair{#1,#2}}(#3)}}
\newcommand{\V}[1]{{\langle{}{#1}\rangle}}
\newcommand{\pair}[1]{{\langle{}{#1}\rangle}}
\newcommand{\VV}[1]{${\langle{}{#1}\rangle}$}

\newcommand{\ifthenelse}[3]{{{\rm{if}}\;{#1}\;{\rm{then}}\;{#2}\;{\rm{else}}\;{#3}\;{\rm{fi}}\,}}
\newcommand{\Int}{\mathbb{Z}}

\newcommand{\indices}[2]{#1[[#2]]}

\newcommand{\name}[2]{{\rm\bf{[{#2}]}}\label{#1}}
\newcommand{\wrt}{{{w.r.t.}}}
\newcommand{\coh}[1]{{\bf{coh}}\{{#1}\}}
\newcommand{\alt}{\;\;{\bf{|}}\;\;}
\newcommand{\sem}[2]{{[\![{#1}]\!]_{{#2}}}}
\newcommand{\select}{{{\mathit{select}}}}
\newcommand{\definedAs}{{\;\;\stackrel{{\mathit{def}}}{=}\;\;}}
\newcommand{\eg}{{\em{e.g.}}}
\newcommand{\ie}{{\em{i.e.}}}
\newcommand{\rng}[1]{ \{ 0\ldots{{#1}}-1 \} }
\newcommand{\Rng}[2]{ \{ {{#1}}\ldots{{#2}} \} }

\begin{document}
\begin{frontmatter}

\title{Generalized Support and Formal Development of Constraint Propagators}
\author[]{\fnms{James} \snm{Caldwell}}
\address{Department of Computer Science, University of Wyoming, 1000 E. University Ave., Laramie, WY 82071-3315, USA\\
E-mail: jlc@cs.uwyo.edu}
\author[]{\fnms{Ian P.} \snm{Gent}}
\address{School of Computer Science, Jack Cole building, University of St Andrews, St Andrews, Fife KY16 9SX, UK\\
E-mail: ian.gent@st-andrews.ac.uk}
\author[]{\fnms{Peter} \snm{Nightingale}}
\address{School of Computer Science, Jack Cole building, University of St Andrews, St Andrews, Fife KY16 9SX, UK\\
E-mail: pwn1@st-andrews.ac.uk}

\begin{abstract}
Constraint programming is a family of techniques for solving combinatorial problems, where
the problem is modelled as a set of decision variables (typically with finite domains) and
a set of constraints that express relations among the decision variables. 
One key concept in constraint programming is \textit{propagation}: reasoning 
on a constraint or set of constraints to derive new facts, typically to remove 
values from the domains of decision variables. Specialised propagation algorithms (propagators) exist
for many classes of constraints.

The concept of {\em support} is pervasive in the design of propagators. 
Traditionally, when a domain value ceases to have support, it may
be removed because it takes part in no solutions. Arc-consistency algorithms
such as AC2001 \cite{bessiere-regin-ac2001} make use of support in the form of
a single domain value. GAC algorithms such as GAC-Schema
\cite{bessiere-gac-schema} use a tuple of values to support each literal. We
generalize these notions of support in two ways. First, we allow a set of
tuples to act as support. Second, the supported object is generalized from a
set of literals (GAC-Schema) to an entire constraint or any part of it.

We design a methodology for developing correct propagators using generalized
support. A constraint is expressed as a family of support properties, which may
be proven correct against the formal semantics of the constraint.  Using
Curry-Howard isomorphism to interpret constructive proofs as programs, we show
how to derive correct propagators from the constructive proofs of the support
properties. The framework is carefully designed to allow efficient algorithms
to be produced. Derived algorithms may make use of {\em dynamic literal
triggers} or {\em watched literals} \cite{Gent_Jefferson_Miguel06} for
efficiency. Finally, two case studies of deriving efficient algorithms are
given.

\end{abstract}


\end{frontmatter}

\section{Introduction}

In this paper we provide a formal development of the notion of support in constraint satisfaction.    This notion is ubiquitous and plays a vital role in the understanding, development, and implementation of constraint propagators, which in turn are the keystone of a successful constraint solver.    While we focus on 
a formal development in this paper, our purpose is not to describe formally what is currently seen in constraint satisfaction.   Instead, we generalize the notion of support so that it can be used in a wider variety of propagators.   The result is the first step in a twin programme of developing a formal understanding of constraint algorithms, while also developing notions such as generalized support which should lead to improved constraint algorithms in the future.  

The methodology presented here for formal development of propagators is based
on the proofs-as-programs and propositions-as-types interpretations of
constructive type theory \cite{Constable_naive,Girard}. 
Like the earlier development in \cite{caldwell_gent_underwood}, the approach 
presented here uses a constructive type theory as the formal framework for 
specifying and developing programs.  There, the proofs were mechanically 
checked in the Nuprl theorem prover \cite{Nuprl}, here the development is 
formal but proofs have not been mechanically checked.

\subsection{Overview of the Constraint Satisfaction Problem}

A constraint is simply a relation over a set of variables. Many different kinds
of information can be represented with constraints. The following are simple
examples: one variable is less than another; a set of variables must take
distinct values; task A must be scheduled before task B; two objects may not
occupy the same space. It is this flexibility which allows constraints to be
applied to many theoretical, industrial and mathematical problems.

The classical constraint satisfaction problem (CSP) has a finite set of 
variables, each with a finite domain, and a set of constraints over those 
variables. A solution to an instance of CSP is an assignment to each variable, 
such that all constraints are simultaneously satisfied --- that is, they are all 
true under the assignment. Solvers typically find one or all solutions, or prove 
there are no solutions. The decision problem (`does there exist a solution?') is 
NP-complete \cite{apt-constraint-programming}, therefore there is no known 
polynomial-time procedure to find a solution.

\subsection{Solving CSP}

Constraint programming includes a great variety of domain specific and general 
techniques for solving systems of constraints. Since CSP is NP-complete, most 
algorithms are based on a search which potentially explores an exponential 
number of nodes. The most common technique is to interleave splitting and 
propagation. Splitting is the basic operation of search, and propagation 
simplifies the CSP instance. Apt views the solution process as the repeated 
transformation of the CSP until a solution state is reached 
\cite{apt-constraint-programming}. In this view, 
both splitting and propagation are transformations, where propagation simplifies 
the CSP by removing domain values that cannot take part in any solution. A splitting 
operation transforms a CSP instance into two or more simpler CSP instances, and 
by recursive application of splitting any CSP can be solved. 

Systems such as Choco~\cite{chocosolver}, IBM ILOG CPLEX CP Optimizer~\cite{ilogsolver} and Minion~\cite{gent-minion-2006,Gent_Jefferson_Miguel06} 
implement highly optimized constraint solvers based on search and propagation,
and (depending on the formulation) are able to solve extremely large problem
instances quickly. 

Our focus in this paper is on propagation algorithms. A propagation algorithm
operates on a single constraint, simplifying the containing CSP instance by
removing values from variables in the scope of the constraint. Values which
cannot take part in any solution are removed. For example, a propagator for
$x\le y$ might remove all values of $x$ which are greater than the largest
value of $y$. Typically propagation algorithms are executed iteratively until
none can make any further simplifications.

\subsection{Proofs to propagators}

Researchers frequently invent new algorithms and (sometimes) give proofs of
correctness, of varying rigour. 
In this paper we 
provide a formal semantics of CSP. 
This allows us to formally characterize 
correctness of constraint propagators, and therefore aid the proof of correctness
of propagators.   
Following this, we lay the groundwork for automatic generation of correct
propagators. The method is to write a set of \textit{support properties}
which together characterize the constraint. Each property is inserted into
a schema, and a constructive proof of the schema is generated. This proof
is then translated into a correct-by-construction propagator.
This method is based on the concept of \textit{generalized support}, 
described in the next section.   
Finally, we give examples of this method by deriving propagators for
the \texttt{element}, \texttt{occurrenceleq} and \texttt{occurrencegeq} constraints. 

\subsection{Generalized support}\label{sub:intro-generalized-support}

Central to this work is the notion of support. This notion is used informally in many
places (for example, in the description of the algorithm GAC-Schema \cite{bessiere-gac-schema}) and more formally by Bessi\`ere 
\cite{BessiereHandbook}. We generalize the concept of support, and develop a formal 
framework to allow us to produce rigorous proofs of the correctness of 
propagators that exploit the generalized concept of support. 

Support is a natural concept in constraint programming. Constraint propagators
remove unsupported values from variable domains, thus simplifying a CSP instance.
Supported values cannot be removed, since they may be contained in a solution.
Thus a support is evidence that a value (or set of values) may be contained in a
solution. If no support exists, it is guaranteed that a value (or set of values) is not contained
in any solution.  

A \textit{support property} characterises the supports of a particular value
(or set of values) for a particular constraint. For example,
three support properties of an element constraint are given by Gent et al. 
\cite{Gent_Jefferson_Miguel06}. Each of these three properties is used to 
create a propagator, such that the three propagators together achieve 
generalized arc consistency. In this instance, writing down support properties
assisted in proving the propagators correct.   

We show that correct support properties can be used to create propagators that
are correct by construction. We describe a general ``propagation schema'', which is a description of
what should be proved when support is lost for a given support property.  
This captures how propagators work in practice.  They are ``triggered'' when it is noted that the current support is lost.
The propagator then seeks to re-establish support.  This might be possible on the current domains, or it may need to narrow domains 
(i.e. remove some values of some variables), or it may be that no new support is possible and the constraint is guaranteed to be 
false.    
The  propagation schema specialised for a given support property 
can be proven constructively. The proof contains sufficient information to be
translated into a correct propagator. We envisage two main uses for such
a propagator. For some constraints, it may be an efficient propagator that
can be used directly. Otherwise, the constructed propagator may be used as part of
an informal argument for the correctness of an efficient propagator.

\subsection{Related Work}

There are a number of items of related work with related or similar goals, however the
approach taken in each case is quite different to our approach. 
Apt and Monfroy~\cite{apt-monfroy-auto-99} generate propagation rules such as
$X=s \rightarrow y\ne a$, where $X$ is a vector of CSP variables, $s$ is a vector
of values within the initial domain of $X$, $y$ is a CSP variable and $a$ is a value 
in the initial domain of $y$. Rules correspond directly to propagation in a 
constraint solver (\textit{ie} when $X$ is assigned $s$, $a$ is removed from the domain of $y$).
A set of rules is generated for a given constraint by a search over the
(potentially very large) space of possible rules. In contrast, our approach is 
much broader in that it is not restricted to generating implication rules. Our
framework allows both the derivation of new propagators and 
proof of correctness of existing ones. 

Beldiceanu, Carlsson and Petit~\cite{beldiceanu-etal-deriving-04} describe constraints
using finite state automata extended with counters. For a constraint $C$, the automaton 
for $C$ can check whether any given assignment satisfies $C$. Beldiceanu, 
Carlsson and Petit give a method to translate an automaton into a set of 
short constraints (a decomposition) such that propagating them will propagate the original constraint
$C$, and there are (in some cases) guarantees of the strength of propagation. 
The approach has been subsequently refined, for example by linking overlapping prefixes and 
suffixes of constraints~\cite{beldiceanu-etal-linking-14}.
Their approach generates decompositions of a particular form,
whereas in this paper our focus is on deriving efficient propagators.  

Jefferson and Petrie~\cite{jeff-petrie-15} studied the properties of triggers, in particular
comparing static triggers with movable triggers on a number of constraint classes
and consistencies. They demonstrate that movable triggers can lead to much more 
efficient propagators. To do this they generalise the concept of support in a
similar way to us, however their work treats each propagator as a monolithic black box
whereas we are interested in constructing propagators and proving correctness and
other properties of them. 

\section{Definitions and Notation}

\subsection{The Standard Mathematical Account}

We start by giving the standard definition of a constraint satisfaction problem
(\eg~ see \cite{FreuderHandbook,BessiereHandbook}).  Formal definitions of the
notations used here are given below.
\begin{definition}[Constraint Satisfaction Problem]
\label{def:csp}
A {\em{Constraint Satisfaction Problem}} (CSP) is given by a triple
$\V{X,\sigma,C}$ where $X$ is a $k$-tuple of variables $X=\V{x_1, \cdots,x_k}$
and $\sigma$ is a {\em{signature}} (a function 
$\sigma:\:X \rightarrow 2^{\Int}$ mapping variables in $X$ to
 their corresponding domains, such that $\sigma(x_i)\subsetneq\Int$ is the
 finite domain of variable $x_i$.)  $C$ is a tuple of extensional constraints
 $C=\V{C_1,\cdots,C_m}$ where each $C_i$ is of the form \VV{Y,R_Y} where
 $Y\subseteq{}X$ is a tuple of variables called the {\em{schema}} or
 {\em{scope}} of the constraint $C_i$.  Also, $R_Y$ is a relation given by a
 subset of the Cartesian products of the domains of the variables in the scope
 $Y$ and is called the {\em{extension}} of $C_i$.
\end{definition}
\begin{definition}[Satisfying tuple]
We say a $Z$-tuple $\tau$ {\em{satisfies}} constraint $\V{Y,R_Y}$ if
$Y\subseteq{}Z$, and the projection $Y[\tau]$ is in $R_Y$ (\ie~if the
projection of the scope $Y$ from $\tau$ is in $R_Y$).
\end{definition}
\begin{definition}[Solution]
\label{def:solution}
A {\em{solution }} to a CSP $\,\V{X,\sigma,C}$ is a tuple $\tau$, with schema
$X$, such that $\tau$ satisfies every constraint in $C$.
\end{definition}

\subsection{Variable Naming Conventions, Ranges, and Literals}

We use lower case letters (possibly subscripted or primed) from near the end of
the Latin alphabet $\{w,x,y,z\}$ to denote variables.  We use Latin letters
$\{i,j,k\}$ to denote integer indexes, and use the Latin letters occurring
early in the alphabet $\{a,b,c,d\}$ (possibly subscripted) to denote arbitrary
integer values.

Ranges are defined as follows.
\[\Rng{b}{c} \definedAs \{a\in\Int\alt b\le{}a \wedge a\le{}c\}\]
We write $2^A$ to denote the powerset (set of all subsets) of $A$.
\comment{
\begin{definition}\name{def:literal}{literal}
A {\em{literal}} is a variable-value pair (\eg~$\V{x,5}$).
\end{definition}
}
A {\em{literal}} is a variable-value pair (\eg~$\V{x,5}$).

\subsection{Vectors}
\label{sec:vectors}
We use uppercase letters $W,X,Y,Z,...$ to denote vectors of variables. We use
the Greek letters $\{\tau,\tau',\tau_1,\tau_2\cdots\}$ to denote tuples of
integer values.

We write finite vectors as sequences of values enclosed in angled brackets,
(\eg~$\V{x,y,z}$).  The empty vector is written $\V{}$.  We take the operation
of prepending a single element to the left end of a vector as primitive and
denote this operation $x\cdot{}Y$.  We abuse this notation by writing
$X\cdot{}Y$ for the concatenation of vectors $X$ and $Y$.  We write $|Y|$ to
denote the length of vector $Y$.  Given a vector $Y$, we write $Y[i]$ to denote
the (zero-based) $i^{th}$ element of $Y$.  This operation is undefined if
$i\not\in\rng{|Y|}$.  

\comment{
\begin{definition}\name{def:membership}{membership in a vector}
\[z\in{}Y \definedAs \exists{}i:\rng{|Y|}.\,Y[i]=z\]
\end{definition}
} Membership in a vector is defined as follows. \[z\in{}Y\definedAs\exists{}i:\rng{|Y|}.\,Y[i]=z\]  We will sometimes need to collect the set of
indexes to an element in a vector.
\[\indices{Y}{z} \definedAs \{i\in\rng{|Y|} \alt Y[i] = z\}\]
\comment{
\begin{definition}\name{def:memindexes}{membership indexes}
\[ Y[z] \definedAs \{i\in\rng{|Y|} \alt Y[i] = z\} \]
\end{definition}
}
Thus, $\indices{\V{x,y,z,x}}{x} = \{0,3\}$.  Note that $\indices{Y}{z}\not=\emptyset$ iff $z\in{}Y$ and
also each index in $\indices{Y}{z}$ is a witness for $z\in{}Y$.

If $y\in{}Y$, we write $Y-y$ to denote the vector obtained from $Y$ by deleting
the leftmost occurrence of $y$ from $Y$. $Y-y = Y$ if $y\not\in{}Y$.  We write
$Z-Y$ for the vector obtained by removing leftmost occurrences of all
($y\in{}Y$) from $Z$.  Given a vector $Z$, we write $\{Z\}$ to denote the set
of values in $Z$ and given a set of variables $S$ we write $\V{S}$ to denote a
vector of the variables in $S$; the reader may assume the variable in $\V{S}$
occur in increasing lexicographic order.  Intersection and unions are defined
on vectors by taking them as sets: $X\cap{}Y \definedAs \V{\{X\}\cap\{Y\}}$;
and $X\cup{}Y \definedAs \V{\{X\}\cup\{Y\}}$.
We write $Y\subseteq{}X$ to mean $\{Y\}\subseteq\{X\}$, \ie~that every element
in $Y$ is in $X$ with no stipulations on relative lengths of $X$ or $Y$ or on
the order of their elements.


\subsection{Signatures}

A signature $\sigma$ is a function mapping variables in $X$ to their associated
domains.  Thus, signatures are functions $\sigma:\:X \rightarrow
2^{\Int}$ where in practice, the subset of integers mapped to is finite.
Where $\sigma$ and $\sigma'$ are signatures mapping variables in $X$ to their
finite integer domains:
\[\sigma'\sqsubseteq_X\sigma \definedAs \forall{}x\in{}X. \sigma'(x)\subseteq \sigma(x)\]
We write $\sigma'\sqsubset_X\sigma$ if $\sigma'\sqsubseteq_X\sigma$ and
$\exists{}x\in{}X:\sigma'(x)\subsetneq\sigma(x)$, \ie~if some domain of
$\sigma'$ is a proper subset of the corresponding domain of $\sigma$.  We drop
the schema subscript when the schema is clear from the context.  We state the
following without proof.
\begin{lemma}[Signature Inclusion Well-founded]
The relation $\sqsubset$ is well-founded if restricted to signatures with
finite domains.
\end{lemma}

\subsection{Relations}
In the description of a CSP given above, a constraint $\V{Y,R_Y}$ is a relation
where the schema $Y$ gives the variable names and $R_Y$ is the set of tuples
in the relation.

Given a signature $\sigma$ mapping variables in schema $Y$ to their domains, a
relation $\V{Y,R_Y}$ is {\em{well-formed}} with respect to $\sigma$ iff the
following conditions hold:
\renewcommand{\labelenumi}{\roman{enumi}.}
\begin{enumerate}
\item All tuples in $R_Y$ have length $|Y|$
\item The values in each column come from the specified domain for that column:
\end{enumerate}
\[\forall\tau:{}R_Y.\; \forall i:\rng{|Y|}.\, \tau[i]\in{}\sigma(Y[i])\]

Schemata are vectors of variable names with no restriction on how many times a
variable may occur.  Thus it is possible to have a wellformed relation whose
schema has common names for multiple columns.  
Given a signature $\sigma$ over a schema $X$, a tuple $\tau$ is called a
$X$-tuple if $\V{X,\{\tau\}}$ is well-formed \wrt~$\sigma$. In this case, we
write $\tuple{X}{\sigma}(\tau)$.  We write $\tuple{X}{\sigma}$ for the set of
tuples satisfying this condition.



\subsubsection{Tuple Coherency}

Conceptually, relations provide a representation for storing valuations
(assignments of values to variables) and so we must distinguish between tuples
which represent coherent valuations (even when their schemata may contain
duplicate variable names) and tuples that do not.  This motivates the following
definitions.

The wellformedness condition on relations requires values in columns labeled
by a variable come from the domain of that variable, but does not rule out
cases where a single tuple with multiple columns named by the same variable
have different values in those columns.

\begin{example}
Consider the relation 
\[\V{\V{x,x,y},\{\V{1,2,3},\V{1,1,3},\V{2,2,3}\}}\]  
The variable $x$ occurs twice in the schema and the first tuple in the schema
assigns different values to $x$, this tuple is not coherent.
\end{example}
An $X$-tuple $\tau$ is {\em{coherent \wrt~variable $z$}} iff the following
holds.
\[\coh{X,z}(\tau) \definedAs \forall{}i,j:\indices{X}{z}. \;\; \tau[i]=\tau[j]\]
We say a tuple is {\em{incoherent \wrt~$z$}} if it is not coherent.
Note that this definition is sensible whether $z\in{}X$ or not.
A simple consequence of the definition is that an $X$-tuple $\tau$ is incoherent \wrt~variable $z$ iff
\[\exists{}i,j:\indices{X}{z}.\;\; \tau[i]\not=\tau[j]\]
An $X$-tuple $\tau$ is {\em{coherent with schema $Y$}} iff it is coherent
\wrt~all variables $z\in{}Y$.
\[\coh{X,Y}(\tau) \definedAs \forall{}z\in{}Y.\;\;\coh{X,z}(\tau)\]
We say an $X$-tuple is \textit{incoherent with respect to schema $Y$} if it is
not coherent \wrt~$Y$.
Only coherent tuples count as solutions (Def.~\ref{def:solution}).

\begin{remark}
In many constraint solvers, incoherent tuples may arise during a computation, but they are
never counted among solutions.  For example, the Global Cardinality constraint 
\[\mathtt{GCC}(\V{x,x,y}, \V{1,2}, \V{(2\ldots 3),(1\ldots 2)})\]
(stating that value 1 occurs two or three times, and value 2 occurs once or twice among variables $\V{x,x,y}$) could 
generate the incoherent tuple $\V{1,2,1}$ internally when using R\`egin's 
algorithm \cite{regin-gcc-96}.\footnote{
R\`egin's algorithm \cite{regin-gcc-96} is polynomial-time and enforces GAC iff 
the schema contains no duplicate variables. 
With duplicate variables, enforcing GAC on GCC is NP-Hard \cite{tractability-globals-04}, therefore it is sensible to use R\`egin's algorithm in this case even though it will not enforce GAC.}
\comment{\footnote{This behaviour
has been observed in Eclipse 5.10 with \texttt{ic::alldifferent} \cite{eclipsesolver} and 
Minion 0.7.0 with \texttt{gacalldiff} \cite{gent-minion-2006}}.}
Generating incoherent tuples
affects both the internal state of a constraint propagator, and the number
of vertices in the search tree.

Strictly speaking, because incoherent tuples
do not count as solutions, the semantics could be specified simply disallowing
them.  However, this approach would rule out faithful finer grained
representations of the internal states of constraint solvers which do generate
incoherent tuples \eg~when searching for support.  Based on this, we have
decided to include them although this adds some complexity to the specification.

\end{remark}

\subsubsection{Selection}

{\em{Selection}} is an operation mapping relations to relations generating new
ones from old by filtering rows (tuples) based on predicates on the values in the tuple.

Given a relation $\V{Y,R_Y}$ and an index $i\in{}\rng{|Y|}$, and a value (say
$a$), {\em index selection} is defined as follows.
\[\select_{(i=a)}(R_Y) \definedAs \{\tau\in{}R_Y | \tau[i] = a\}  \]

The tuples selected from a relation by index selection are not guaranteed to be
coherent with respect to schema $Y$.

Given a relation $\V{Y,R_Y}$, a variable $x$, and a value $a$, {\em value selection} is defined as follows.
\[\select_{(x=a)}(R_Y) \definedAs\]
\[\qquad \qquad \{\tau\in{}R_Y \: | \: \forall i:\indices{Y}{x}.\;\; \tau[i] = a\}\]

Thus a tuple $\tau$ is included in a selection $\select_{(x=a)}R_Y$ if and only if
all columns of $\tau$ indexed by $x$ have value $a$, \ie~$\tau$ must be
coherent for $x$ and those columns must have value $a$.  

\begin{lemma}\name{lemma:wf-select}{Selection Wellformed}
For all well-formed relations $\V{Y,R_Y}$ and all $x$, and all
$a\in\Int$, the relation $\V{Y,\select_{(x=a)}R_Y}$ is well-formed.
\end{lemma}

Finally, we define {\em coherent selection} as follows. 
\[\select_Y(R_X) \definedAs \{\tau\in{}R_X \: |\: \coh{X,Y}(\tau)\} \]
Coherent selection selects the tuples which are coherent with respect to $Y$.



\subsubsection{Projection}
Projection is an operation for creating new relations from existing ones by
allowing for the deletion, reordering and duplication of columns.  We use a
generalized version here that allows duplicate names. This is because 
many constraint solvers (including Minion \cite{gent-minion-2006} for example) allow 
schemata to contain duplicate names.

\begin{lemma}\name{lemma:proj-map}{Projection maps exist}
 For all vectors $X$ and $Y$, if $Y\subseteq{}X$, then there exists a function
 from the indexes of $Y$ to the indexes of $X$ (say $f\in\rng{|Y|}
 \rightarrow\rng{|X|}$) such that
\[\forall{}i:\rng{|Y|}.\; Y[i] = X[f(i)]\]
\end{lemma}

Note that there is no restriction on the relative lengths of $X$ and $Y$, \eg~
it is possible for any of the following to hold: $|Y|< |X|$, $|Y| = |X|$ or
$|Y| > |X|$.  The projection maps are evidence witnessing claims of the form
$Y\subseteq{}X$.  Furthermore, because our model allows for duplicated columns,
there may be multiple projection maps witnessing an inclusion $Y\subseteq{}X$.

\begin{example}
Consider 
\[\begin{array}{ll}
Y = \V{x_4,x_2,x_2,x_1,x_3} & X=\V{x_1,x_2,x_3,x_4}
\end{array}\] 
then $Y\subseteq{}X$ is witnessed by the projection map:
\[\{\V{0, 3},\V{1,1},\V{2,1},\V{3,0},\V{4,2}\}  \]
Similarly, $X\subseteq{}Y$ and is witnessed by the following.
\[ \{\V{0, 3},\V{1, 1},\V{2, 4},\V{3, 0}\}\]
Also $\V{x_2}\subseteq{}Y$ is witnessed by two functions, \(\{\V{0,1}\}\) and \(\{\V{0,2}\}\).
\end{example}

\begin{lemma}\name{lemma:tuple-proj}{Tuple Projection}
Given $X$ and $Y$, if $Y\subseteq{}X$ is witnessed by $f$, for each $X$-tuple
$\tau$ there is a vector $Y_f(\tau):\rng{|Y|}\rightarrow\Int$ such that
\[\forall{}i:\rng{|Y|}.\: Y_f(\tau)[i] = \tau[f(i)]\]
\end{lemma}

\begin{corollary}\name{lemma:tuple-proj-wf}{Tuple Projection Wellformed}
Given $X$ and $Y$, if $Y\subseteq{}X$ is witnessed by $f$, for each $X$-tuple
$\tau$, $Y_f(\tau)$ is a $Y$-tuple,
\ie~$|Y_f(\tau)| = |Y|$ and all values
in $Y_f(\tau)$ are in their domains.
\end{corollary}

Whenever $Y\subseteq{}X$, projection maps $f$ and $g$ witnessing this fact
behave the same when used to index into tuples coherent with $Y$.  This is
illustrated by the following example.

\begin{example}\label{example:projection-map}
Suppose $Y=\V{x,y}$ and $X=\V{x,x,w,y,w}$ then there are two projections maps
witnessing $Y\subseteq{}X$, $f=\{\V{0,0},\V{1,3}\}$ and
$g=\{\V{0,1},\V{1,3}\}$. Now, any length $|X|=5$ tuple coherent with $Y$ is
of the form $\tau=\V{a,a,b,c,d}$ where $a,b,c,d\in{}\Int$.  Thus, even though
$f(0)\not=g(0)$ the following equalities hold:
\[\tau[f(0)] = \tau[0] = a = \tau[1] = \tau[g(0)]\]
\end{example}

This observation is made precise by the following lemma.

\begin{lemma}\name{lemma:coherent-proj-unique}{Coherent Projection Unique}
For all $X$ and $Y$, and for all projection maps $f$ and $g$ witnessing
$Y\subseteq{}X$, for all $X$-tuples $\tau$ coherent with schema $Y$,
$Y_f(\tau) = Y_g(\tau)$.
\end{lemma}
\comment{
\begin{proof}
To show this, we use Extensionality (Def.~\ref{fact:ext}).  Choose arbitrary
$i\in\rng{|Y|}$ and show $Z_f(\tau)[i] = Z_g(\tau)[i]$.  To show this, we must
show $\tau[f(i)] = \tau[g(i)]$. Since $f$ and $g$ are projection maps
witnessing $Y\subseteq{}X$, we know that both $f(i),g(i)\in\rng{|X|}$ so they
are in range to index $\tau$.  Either $f(i)=g(i)$ or not.  In the first case
the equality $\tau[f(i)]=\tau[g(i)]$ holds trivially.  Consider the case where
$f(i)\not=g(i)$. Since $f$ and $g$ witness $Y\subseteq{}X$ we know that
$X[f(i)] = X[g(i)]$ \ie~ they index a common variable (say $z$) in X. Also,
(since they are projection maps) $Y[i] = z$.  Now, since $\tau$ is coherent
with respect to schema $Y$ and therefore is coherent with respect to variable
$z$ in particular, $\tau[f(i)]=\tau[g(i)]$.
\end{proof}
}
\begin{notation}
\label{note:pmap}
Since projections $Z$ where $Z\subseteq{}X$ do not depend on the projection map
they are built from when the $X$-tuple $\tau$ is coherent with $Z$, we
will simply write $Z(\tau)$ in this case.
\end{notation}

\begin{lemma}\name{lemma:proj.con}{Projection Coherent}
For all $X$, $Y$ and $Z$, if $Y\subseteq{}X$ and if $\tau$ is an $X$-tuple
coherent with $Z$, then $Y[\tau]$ is a $Y$-tuple coherent with $Z$.
\end{lemma}

So far we have defined projection of a single tuple, potentially with repeated 
variables in the schema. We lift the notation tuple-wise to relations as given by the following
definition.

\begin{definition}\name{def:relation-proj}{Relation Projection}
Given $X$ and $Y$, and a wellformed relation $\V{X,R_X}$, if $Y\subseteq{}X$ is
witnessed by $f$, 
\[Y_f(\V{X,R_X}) =\]
\[\qquad \V{Y, \{\tau \in \Int^{|Y|} \alt \exists{}\tau'\in{}R_X.\;\; \tau = Y_f(\tau')\}}\]
\end{definition}

\begin{lemma}\name{lemma:relation-proj-wf}{Relation Projection WF}
For all well-formed relations $\V{X,R_X}$ and all $Y$, $Y\subseteq{}X$ having a
projection map $f$, the relation $Y_f(\V{Y,R_Y})$ is well-formed.
\end{lemma}

\subsubsection{Equivalence of Constraints}

Now that we have relation projection, we are able to define an equivalence of 
constraints which does not depend on the ordering (or the length) of schemata. 

\begin{definition}\name{def:schema-equivalence}{Schema Equivalence}
\[
X \equiv Y \definedAs  X \subseteq Y \wedge Y \subseteq X
\]
\end{definition}

Schema equivalence requires only that \(X\) and \(Y\) contain the same set of variables. 
The order of variables and the number of duplicates are not restricted.

\begin{definition}\name{def:constraint-equivalence}{Constraint Equivalence}
\[
\begin{array}{l}
\V{X, R_X} \equiv \V{Y, R_Y} \definedAs \\ 
\abit{1} X \equiv Y \wedge \\
\abit{1} Y \subseteq X \textrm{ is witnessed by projection map }f \wedge \\
\abit{1} Y_{f}( select_X(\V{X, R_X})) = \\
\abit{1} \abit{1} \select_Y(\V{Y, R_Y})
\end{array}
\]
\end{definition}

There are several steps to the constraint equivalence definition. First, it is required that the schemata are equivalent. 
Then we find a projection map \(f\) that will be used to reorder the schema \(X\) to match \(Y\).
Coherent selection is used to remove the incoherent tuples of both constraints. The schema \(X\) of the first
constraint is reordered to match \(Y\). 
Finally, the two constraints are equivalent if they have the same set of coherent tuples.

Incoherent tuples are removed \textit{before} reordering the schema \(X\), therefore
any projection map \(f\) will produce the same set of reordered tuples (as in Example~\ref{example:projection-map}).

\subsection{Syntactic Definition of Relations}

Constraints are rarely presented extensionally but are instead described in
some syntactic way.  We introduce the following notation to denote the map from
syntactic descriptions to their extensional meanings.
\begin{definition}[Semantics]
Given a syntactic description of a constraint (say ${\cal{C}}$) over schema $X$
and where $\sigma$ is a signature consistent with $X$, we will write
$\sem{{\cal{C}}}{\sigma}$ to denote its extension.
\end{definition}

So, if we have a constraint \(\texttt{Element}(X,y,z)\) where 
\(X\) is a vector of variables and \(y\) and \(z\) are variables, and \texttt{Element} has 
a defined meaning, we can write \(\sem{\texttt{Element}(X,y,z)}{\sigma}\) to 
obtain its relation within some signature \(\sigma\). 

\section{Propagation and Support}

Propagation is the process of narrowing the domains of variables so that
solutions are preserved. This effectively shrinks the search-space and is one
of the fundamental techniques used in constraint programming.  It has been
described (\cite[pp.~17]{FreuderHandbook}) as a process of {\em{inference}} to
distinguish it from {\em{search}}. Most work on propagation considers the constraints 
singly

\begin{definition}\name{def:gac}{Generalized Arc Consistency}
Given a constraint $\mathcal{C}$ with schema $X$ and a signature $\sigma$, we say
$\sigma'\sqsubseteq{}\sigma$ is {\em{Generalized Arc Consistent}} iff
\[\begin{array}{l}
\forall{}i\in{}\rng{|X|}.\; \forall{}a\in\sigma(X[i]).\\
\qquad a\in \sigma'(X[i]) \leftrightarrow \exists{}\tau\in{}\sem{{\cal{C}}}{\sigma}.\;\tau[i]=a\\
\end{array}
\]
If $\sigma'$ is Generalized Arc Consistent, we say it is {\em{GAC}}.
\end{definition}

\begin{corollary}\name{lemma:gac}{Generalized Arc Consistency}
Given a constraint ${\cal{C}}$ and a signature $\sigma$, $\sigma$
is GAC for ${\cal{C}}$ iff
\[\begin{array}{ll}
 & \forall{}\sigma'\sqsubset\sigma. \; \sem{{\cal{C}}}{\sigma'}\subset
\sem{{\cal{C}}}{\sigma} \\
\end{array}\]
{\em{i.e.}} if all signatures having strictly narrower domains provide strictly
fewer solutions for $\mathcal{C}$ than $\sigma$.
\end{corollary}

Enforcing GAC is the strongest form of propagation that considers constraints singly and acts 
only on the variable domains. 
Other forms of consistency (such as bound consistency) lie between GAC and no change (i.e.\ \(\sigma'=\sigma\)). 


\comment{
Given a syntactic description of an arithmetic constraint $\mathcal{C}$ over
schema $X$, we use $\sem{\mathcal{C}}{\mathbb{R}}$ to denote the set of all 
satisfying tuples of $\mathcal{C}$ where the values in the tuples are drawn from
$\mathbb{R}$. 

\begin{definition}\name{def:boundsr}{Bounds($\mathcal{R}$)-consistency}
Given a constraint $\mathcal{C}$ with schema $X$ and a signature $\sigma$, we say
$\sigma'\sqsubseteq{}\sigma$ is {\em{Bounds($\mathcal{R}$)-Consistent}} iff
\[\forall{}i\in{}\rng{|X|}.\; \forall{}k\in\{\mathrm{min}(\sigma'(X[i])), \mathrm{max}(\sigma'(X[i]))\}.\;\exists{}\tau\in{}\sem{{\cal{C}}}{\mathbb{R}}.\;\tau[i]=k\]
\end{definition}
}

\subsection{Support}

The concept of support was introduced in Section 
\ref{sub:intro-generalized-support}. Support is {\em evidence} that a set of 
domain values (or a single value) are consistent for some definition of 
consistency (for example, GAC) for a particular constraint $C$. If a set of 
values have no support, then they cannot be part of any solution to $C$, and 
therefore can be eliminated from variable domains without losing any solutions 
to the CSP. The concept of support is central to the process of propagation. 

In \cite[pp.~37]{BessiereHandbook} Bessi\`ere gives a description of when a 
tuple
supports a literal. We use a more expressive model where support (or perhaps we
should call it {\em{evidence}}) is defined by sets of tuples.  In most cases,
supports will be singletons (\ie~they are simply represented by a set
containing a single tuple). However, some constraints require a set of tuples to 
express the condition for support.

\begin{example}\label{ex:alldiff}
Consider the constraint \[\mathcal{C}=\textrm{AllDifferent}(x_1,x_2,x_3)\] with the signature $\sigma$ : 
\(x_1 \in \{1,2\}\), \(x_2 \in \{1,2,3,4\}\), \(x_3\in \{1,2,3,4,5\}\). This 
signature is GAC. Given Bessi\`ere's description of support 
\cite[pp.~37]{BessiereHandbook} (as used by general-purpose GAC algorithms
such as GAC-Schema \cite{bessiere-gac-schema}), each literal in the signature 
would be supported by a tuple containing the literal. Hence every literal is 
contained in the support for $\mathcal{C}$. However, not all literals
are required; the following set
is sufficient: 
\(L=\{ \V{x_1,1}, \V{x_1, 2}, \V{x_2,2}, \V{x_2,4}, \V{x_3,2}, \V{x_3,3},\) \(\V{x_3,5} \}\) \cite[\S 5.2]{nightingale_all_diff}. While all literals
in $L$ remain valid, in some smaller signature $\sigma_1 \sqsubseteq \sigma$,
then the constraint remains GAC. This can be used to avoid calling the 
propagator, and therefore is important to capture in our definition of 
generalized support.
\end{example}

Extensional constraints (sets of tuples) are interpreted disjunctively, \ie~ as
long as the set is non-empty, a solution exists.  Similarly, support exists if
the support set is non-empty.  Our generalization of support is to model it as
a set of tuples interpreted conjunctively {\em{i.e.}} thay all must be valid
for support to exist.  Thus, a generalized support set is a disjunction of
conjunctions ($\exists\forall$); we say support exists if at least one support is present in the set and all the tuples in that support are valid w.r.t.~variable domains.

We use the following as a simple running example throughout this section.

\begin{example}\label{ex:running-example}
Consider the constraint $x+y+z\ge 2$ with initial signature $\sigma$ : $x,y,z \in \{ 0,1\}$. The signature is GAC, and the constraint is satisfied by three
tuples: \[\begin{array}{l}\sem{x+y+z\ge 2}{\sigma}=\\ \qquad\{\langle0,1,1\rangle, \langle1,0,1\rangle, 
\langle 1,1,0\rangle \}\\\end{array}\].
\end{example}

\subsubsection{Support Sets}

\begin{definition}\name{def:property}{Support property}
Given a schema $Y$ and signature $\sigma$ over $Y$, a {\em{support property}}
is a predicate 
\[P:{\mathit{signature}} \rightarrow 2^{\Int^{|Y|}} \rightarrow \mathbb{B}\] 
mapping signatures and sets of integer tuples of length $|Y|$ to a Boolean.  We
will sometimes write the parameter indicating which signature $P[\sigma]$
depends on as a subscript $P_{\sigma}$ or drop it entirely if the property does
not depend on a signature.
\end{definition}


\begin{definition}\name{def:support}{Support Set for a property $P$}
Given a schema $Y$ and a signature $\sigma$ over $Y$ and a property of sets of
$Y$-tuples, $P_{\sigma}$ we define the support set for $P$ to be the set:
\[\begin{array}{l}
\Support{Y}{\sigma}{P} \definedAs\vspace{.06125in}\\
\abit{1}\{S\subseteq{}\tuple{Y}{\sigma} | P_{\sigma}(S) \wedge \forall{}S'\subset{}S.\; \neg P_{\sigma}(S')\}
\end{array}\]
\end{definition}

Note that support sets are minimal \wrt~the property $P$ since they contain
no subset which also satisfies the property.  

Consider example \ref{ex:running-example}, the constraint $x+y+z\ge 2$. 
One support property is the following.
\[\begin{array}{l}P_{\sigma}(S) \definedAs \exists \tau \in S.\\
\qquad \sum \tau \ge 2 \wedge \tau[0]=\min(\sigma(x))\\ \end{array}\]
This property admits
sets of tuples of any size as long as one tuple satisfies the constraint, and
the value for $x$ in that tuple is the minimum value in $\sigma(x)$. 
This support property corresponds to a propagator that prunes the minimum value
of $x$ whenever there is no supporting tuple containing it. To enforce GAC,
two other properties would be required for $y$ and $z$.
The support set for $P_{\sigma}$ is 
$\Support{\langle x,y,z\rangle}{\sigma}{P}=\{\{\langle0,1,1\rangle\} \}$.

A collection of properties is supported if they all are.
\begin{definition}\name{def:collection-support}{Support for a collection of properties}
If ${\cal{P}} =\{P_1,\ldots,P_k\}$ is a collection of
properties sharing schema $Y$ and $\sigma$ is a signature over $Y$, we write
\[\begin{array}{l}
\Support{Y}{\sigma}{{\cal{P}}} \definedAs\\
\qquad \forall{}P\in{\cal{P}}.\: \Support{Y}{\sigma}{P}\ne \emptyset\\
\end{array}\]
\end{definition}

\subsubsection{Admissible Properties and Triggers}

Our language for properties is unrestrained and allows us to specify properties
that are not sensible for specifying propagators. Therefore an admissibility condition is required. We define {\em p-admissibility} as follows.


\begin{definition}\name{def:p-admissible}{P-Admissibility}
We say a property $P$ is {\em{p-admissible}} if it satisfies the following
condition.
\[
\begin{array}{l}
\forall{}\sigma. \:\forall{}\sigma'\sqsubseteq\sigma. \\
 \abit{1}\forall{}S\subseteq\tuple{Y}{\sigma}.\\
 \abit{2} (P_{\sigma}(S) \; \wedge \; S\subseteq\tuple{Y}{\sigma'}) \Rightarrow P_{\sigma'}(S)
\end{array}
\]
In this case, we write ${\mathit{p\!-\!admissible}}(P)$.
\end{definition}

\comment{ P-admissibility is a rather technical condition, and yet it seems
  reasonable that not all properties would be useful for specifying
  propagators.  } 
P-admissibility is a kind of stability condition on
properties that guarantees that if a $P_{\sigma}(S)$ holds and the domain is
narrowed to $\sigma'$, but no tuple is lost from $S$ because of the narrowing,
then $P_{\sigma'}(S)$ must also hold.  In the implementation of
dynamic-triggered propagators \cite{Gent_Jefferson_Miguel06}, it is implicitly 
assumed that support for these propagators satisfy this property.

Continuing example \ref{ex:running-example}, the support property
$P_{\sigma}(S) \definedAs \exists \tau \in S: \sum \tau \ge 2 \wedge 
\tau[0]=\min(\sigma(x))$ is p-admissible: $\sum \tau \ge 2$ does not depend
on $\sigma$, and $\tau[0]=\min(\sigma(x))$ can only be falsified under $\sigma'$ 
when the value $\min(\sigma(x))$ is not in $\sigma'(x)$. This means
$\tau$ is not in $\tuple{\langle x,y,z \rangle}{\sigma'}$ so the implication 
is trivially satisfied.
Suppose $S=\{\langle 0,1,1 \rangle\}$. The only way $P_{\sigma'}(S)$ can be false
is if $0 \notin \sigma'(x)$. In this case, $S$ contains a tuple that is not
valid in $\sigma'$ therefore the p-admissibility property is trivially true.

A constraint solver has a trigger mechanism which calls propagators
when necessary. Each propagator registers an interest in domain events by {\em
  placing triggers}.  For example, if a propagator placed a trigger on $\V{x,
  a}$, then the removal of value $a$ in $\sigma(x)$ would cause the propagator
to be called. (This is named a {\em literal trigger}  \cite{Gent_Jefferson_Miguel06}, or \textit{neq event}
\cite{schultestuckey-propengines-08}.)

In this paper, we focus on literal triggers which can be moved during search. 
We consider two different types of movable literal trigger: those which are 
restored as search backtracks (named {\em dynamic literal triggers}), and those 
which are not restored (named {\em watched literals} 
\cite{Gent_Jefferson_Miguel06}).

The definition of p-admissibility allows the use of dynamic literal triggers, 
among other types. 
Watched literals are preferable to dynamic literal triggers because there is no
need to restore them when backtracking, which saves space and time. However, it
is not always possible to apply watched literals. We define an additional
condition on properties named {\em backtrack stability}, which is sufficient to 
allow the use of watched literals.

\begin{definition}\name{def:backtrack-stable}{Backtrack Stability}
We say a property $P$ is {\em{backtrack stable}} if it satisfies the following
condition.
\[
\begin{array}{l}
 \forall{}S.\:\forall{}\sigma.\:\forall{}\sigma'\sqsubseteq\sigma.\\ 
 \abit{1} S \ne \emptyset \Rightarrow \\
 \abit{2} P_{\sigma'}(S) \Rightarrow P_{\sigma}(S)
\end{array}
\]
\end{definition}

Backtrack stability states that any non-empty support $S$ under $\sigma'$ must remain a support for all signatures $\sigma$ where $\sigma$ is larger than $\sigma'$. This guarantees that a non-empty support $S$ will remain valid as the search backtracks. The empty support indicates that the property is trivially satisfied; this support is not usually valid after backtracking, so it is excluded here. 

Continuing example \ref{ex:running-example}, the support property
$P_{\sigma}(S) \definedAs \exists \tau \in S: \sum \tau \ge 2 \wedge 
\tau[0]=\min(\sigma(x))$
is not backtrack stable because $\min(\sigma(x))$ may not be the same as 
$\min(\sigma'(x))$.

Backtrack stability is in fact too strong: it is not necessary for a support to
remain valid for {\em all} larger signatures, it is only necessary for it to remain valid at signatures that are reachable on backtracking. However it is sufficient for the purposes of this paper. 

Backtrack stability also depends on the form of properties. The element support properties presented in Section \ref{sec:element-support-properties} are not backtrack stable. However, they can be reformulated to be backtrack stable, by dividing them up as we show in Section \ref{sec:element-p-admiss}.

For some property $P_{\sigma}(S)$ the support $S$ is \textit{evidence} that the constraint corresponding to $P$ is consistent.  The intuition is that $S$ remains valid evidence until domains are narrowed to the extent that $S\not\subseteq \tuple{Y}{\sigma'}$ (where $\sigma'\sqsubseteq\sigma$). This is an efficiency measure: a constraint solver can disregard the constraint corresponding to $P$ until $S\not\subseteq \tuple{Y}{\sigma'}$.



For example, the property $P_{\sigma}(S) \definedAs \forall b \not\in
\sigma(j). \V{i,b}\in S$ is not p-admissible when $j\ne i$. 



\begin{definition}\name{True and False properties}{Properties {\it{True}} and {\it{False}}}
We define the constant properties {\it{True}} and {\it{False}} by lifting them
to functions of sets of tuples.
\[
\begin{array}{l}
True(S) = True\\
False(S) = False\\
\end{array}
\]
\end{definition}

\begin{lemma}\name{lemma:true-singleton}{True singleton}
For all $Y$ and for every signature $\sigma$ over $Y$,
\[\Support{Y}{\sigma}{True} = \{\emptyset \} \]
\end{lemma}
Note, that it might be assumed that if any of the domains in $\sigma$ are
empty, then there should be no support, even for the {\it{True}} property.
Checking for emptiness is not a function of support, but is done at a higher
level.

\begin{lemma}\name{lemma:false-empty}{False Empty}
For all $Y$ and for every signature $\sigma$ over $Y$,
\[\Support{Y}{\sigma}{False} = \emptyset \]
\end{lemma}

\begin{corollary}\name{lemma:true-false-admissible}{{\it{True}} and {\it{False}} are p-Admissible}
The properties {\it{True}} and {\it{False}} are p-Admissible.
\end{corollary}

We can  combine supports by taking the conjunctions or disjunctions of their properties.

\begin{definition}
We define the conjunction and disjunction of support properties as follows.
\[\begin{array}{l}
(P\wedge{}Q)_{\sigma}(S) \definedAs P_{\sigma}(S)\wedge{}Q_{\sigma}(S) \\
(P\vee{}Q)_{\sigma}(S) \definedAs P_{\sigma}(S)\vee{}Q_{\sigma}(S) 
\end{array}\]
\end{definition}




We state the following lemma without proof.
\begin{lemma}\name{lemma:conj-p-admissible}{\(\wedge\) and \(\vee\) are p-admissible}
Given a schema $Y$ and signature $\sigma$ for $Y$ and two p-admissible
properties $P$ and $Q$, then $(P\wedge{}Q)$ and $(P\vee{}Q)$
are p-admissible as well.
\end{lemma}

\comment{\begin{conjecture}\name{lemma:disj-p-admissible}{Disjunction p-admissible}
Given a constraint $C$ with signature $X$, if $P_\rho$ where $\rho$ is evidence
for $Y\subseteq{}X$ and $Q_{\rho'}$ where $\rho'$ is evidence for
$Z\subseteq{}X$ are p-admissible properties sets $S,
S\subseteq\tuple{X}{\sigma}$ then $(P\vee{}Q)_{\hat{\rho}}$ is p-admissible as
well, where $\hat\rho(i) = \ifthenelse{i < |X|}{\rho(i)}{\rho'(i)}$ is evidence
for $Y\cdot{}Z\subseteq{}X$.
\end{conjecture}
}

\subsubsection{Extensional Support for Literals}

\begin{definition}\name{def:literal-support}{Support Property (for a Literal)}
Given a schema $Y$, a signature $\sigma$ over $Y$, and a literal $\pair{i=a}$, then: $\pair{i=a}$ denotes the property
supporting this literal and is given by:
\[\pair{i=a}(S) \definedAs \exists{}\tau\in{}S. \;\tau[i]=a  \]
The support set for $\pair{i=a}$ is simply the set \(\Support{Y}{\sigma}{\pair{i=a}}\).
\end{definition}

\begin{corollary}
If $S\in{}\Support{Y}{\sigma}{\pair{i=a}}$ then  $S$ is a singleton.
\end{corollary}
\begin{proof}
Assume $S\in{}\Support{Y}{\sigma}{\pair{i=a}}$ then $\pair{i=a}(S)$ holds,
{\em{i.e.}} we know $\exists{}\tau\in{}S.\tau[i]=a$. Thus $|S|\ge{}1$.  Now, we
assume that $|S|>1$ and show a contradiction. There is at least one tuple in
$S$, such that $\tau[i]=a$. If there is any other tuple $\tau'\in{}S$ where
$\tau\not=\tau'$ then $\pair{i=a}(S-\{\tau'\})$ holds as well, and since this
set is smaller, $S$ was not minimal and so was not a support as we assumed.
\end{proof}

\begin{lemma}\name{lemma:p-admissible-literal}{Literals are p-admissible}
Given a schema $Y$ and a signature $\sigma$ on $Y$, if $i\in\rng{|Y|}$ and
$a\in\sigma(Y[i])$ then $\pair{i=a}$ is a P-admissible property.
\end{lemma}
\begin{proof}
Note that $\pair{i=a}$ does not refer to $\sigma$ at all and so is P-admissible.
\end{proof}

\subsubsection{Structural Support - Evidence}

Literal support captures support for variable-value pairs.  Structural support
is support for some structural condition not representable by a single tuple.
Thus, if any tuple in a structural support is lost, then the support no longer
holds. In example \ref{ex:alldiff} (GAC AllDifferent) we gave a list of literals
as evidence that an AllDifferent constraint is GAC. A list of literals would be 
represented as a structural support in our framework by using the support property
for a literal (for each literal individually) then finding support for a collection 
of properties (as in Defn.~\ref{def:collection-support}).

Constraint solvers typically allow movable triggers to be placed on literals, so the connection
between literals and our definition of generalised support is important for this paper.
A generalised support may be less compact than the set of literals it represents. 
However, the implementation of a propagator
may correctly place triggers on the set of literals. Generalised support is merely
an abstraction used in our framework.  

\comment{
\begin{example}
In \cite{nightingale_all_diff} it is shown how to use strongly-connected
components (SCC) as support for solving the all-different constraint.
Representation of this evidence takes the form of a quasi-bipartite graph
(between variables and the values in their domains) witnessing the existence of
a SCC.  Semantically, the extensional representation of this evidence requires
a set of tuples, all of which must be present if the support is to remain.

\end{example}
}

\subsection{Soundness and Completeness of a Collection of Propagators}

Propagators narrow domains to minimize the search space and provide evidence
that the narrowed domains have not eliminated any solutions. Constraints may be
supported by a collection of propagators.  To show that the propagators are
correct with respect to the constraint they support we show they are sound and
complete.  

\subsubsection{Soundness}

\begin{definition}\name{def:soundness}{Propagator Soundness}
Given a constraint $C$ with schema $Y$ and a set of propagators
${\cal{P}}=\{P_1,\cdots,P_m\}$ we say ${\cal{P}}$ is {\em{sound}} with respect
to the constraint $C$ if the following holds:
\[
\begin{array}{l}
\forall{}\sigma.\;\mathrm{singleton}(\sigma) \Rightarrow \\
\qquad (\Support{Y}{\sigma}{\cal{P}} \Rightarrow \sem{C}{\sigma} \ne \emptyset)
\end{array}
\]
\end{definition}

Soundness says that for the most restricted non-empty signatures (ones where
all domains in the signature have been narrowed to a singleton) the propagator
must be able to distinguish between the constraint being empty or inhabited by
a single tuple.  If support is non-empty at a singleton domain then the
constraint must be true there as well.  The definition of soundness presented
here is related to the one in \cite{Tack_Schulte_Smolka06}.

Thinking of support as evidence for truth, one might expect soundness to be
characterized as follows:
\[\forall{}\sigma.\;\Support{Y}{\sigma}{\cal{P}} \Rightarrow \sem{C}{\sigma} \ne \emptyset\] 
This is too strong. At a non-singleton signature, support is an approximation
to truth.  For example, even though a constraint may fail in a particular
non-convex domain (\ie~the domain has gaps), a propagator that operates on
domain bounds may not recognize the domain is not convex until the signature
has been narrowed further.

\subsubsection{Completeness}

Completeness guarantees that if the meaning of a constraint is non-empty at a
signature $\sigma$ (semantic truth) then there is support for the family of
properties $\cal{P}$. The wrinkle on this scheme is that the support may not
exist at $\sigma$ itself, but only at some refined
$\sigma'\sqsubseteq{}\sigma$.  If so, we insist that the constraint has not
lost any tuples at the refined signature $\sigma'$.

\begin{definition}\name{def:completeness}{Propagator Completeness}
Given a constraint $C$ with schema $Y$ and a set of propagators
${\cal{P}}=\{P_1,\cdots,P_m\}$ we say ${\cal{P}}$ is {\em{complete}} with
respect to the constraint $C$ if the following holds:
\[
\begin{array}{l}
\forall{}\sigma.\;\; \sem{C}{\sigma} \ne \emptyset \;\Rightarrow \\
\abit{1}\exists{}\sigma'\sqsubseteq{}\sigma. \\
\abit{2}\sem{C}{\sigma} \subseteq \sem{C}{\sigma'} \: \wedge \: \Support{Y}{\sigma'}{\cal{P}}
\end{array}
\]
If ${\cal{P}}$ is complete we write ${\mathit{complete}}({\cal{P}})$.
\end{definition}


\begin{theorem}\name{thm:local-completeness}{Local Completeness}
Give a set of properties ${\cal{P}} = \{P_1,\cdots,P_k\}$ defined over schema
$Y$, if each singleton $\{P_i\}$ is complete then ${\cal{P}}$ is complete.
\end{theorem}
\begin{proof}
If ${\cal{P}}$ is supported at $\sigma$, then use witness $\sigma$ for
$\sigma'$ and completeness trivially holds.  Suppose there is not support for
${\cal{P}}$ at $\sigma$ where $\sem{C}{\sigma} \neq \emptyset$.  Choose one of the
$P_i\in{\cal{P}}$ such that $\neg\Support{Y}{\sigma}{P_i}$ and let $\sigma',
\sigma'\sqsubset\sigma$ be the signature claimed to exist in the proof of
completeness of $P_i$. By completeness of $\{P_i\}$,
$\sem{C}{\sigma}\subseteq\sem{C}{\sigma'}$.  If there is support for
${\cal{P}}$ at $\sigma'$ then ${\cal{P}}$ is complete. If not, iterate this
process by choosing another $P_k\in{\cal{P}}$ that is not supported at
$\sigma'$.  The fixed-point of this process must yield a signature
$\hat{\sigma}$ such that $\Support{X}{\hat\sigma}{\cal{P}}$.  The fixed-point
exists because $\sqsubset$ is a well-founded relation on signatures.\end{proof}

Our definition of completeness ensures that a propagator derived from
a support property does not fail early, therefore it is merely a correctness
property. It is similar in intention to Maher's definition of weak completeness
\cite{maher-prop-completeness}, although Maher's definition only applies to
singleton domains.

Soundness and completeness as defined here are the minimum conditions required for
a propagator to operate correctly, thus popular notions of consistency such as GAC, 
bound(\(\mathbb{Z}\)) and bound(\(\mathbb{R}\)) are sound and complete, and therefore 
are supported in our framework. 
Soundness and completeness are satisfied by very simple
support properties such as:

\[\begin{array}{l}
P_\sigma(S) \definedAs (\neg \mathrm{singleton}(\sigma) \rightarrow S\ne \emptyset) \\
\abit{5}\wedge (\mathrm{singleton}(\sigma) \rightarrow \sem{C}{\sigma}\ne \emptyset)
\end{array}
\]

This property corresponds to a propagator that waits until all variables are 
assigned before checking the constraint. Any practical propagator is stronger than
this. 

Soundness and completeness are not the only options for
characterizing the correctness of a set of generalized support properties.  For example, in
\cite{Gent_Jefferson_Miguel06} it is shown that a set of properties imply the domain is
GAC.  Other forms of consistency such as bound consistency could also serve as
correctness conditions for a set of properties.

\subsection{Formal Development of Constraint Propagators}

The methodology for formal development of propagators for a constraint $C$ is
as follows:
\begin{enumerate}

\item{} Describe support properties (${\cal{P}}=\{P_1,\cdots,P_k\}$) that
  characterize constraint $C$ and prove that they are p-admissible.

\item{} For each property $P_i$, give a constructive proof of the propagation
  schema given in Def.~\ref{def:support-framework}.  The computational content
  of these proofs gives correct-by-construction algorithms for each
  propagator.

\item{} Prove the soundness and completeness of ${\cal{P}}$ with respect to
  $C$. This shows the collection of propagators are correct \wrt~the constraint
  $C$.  This proof often reuses the propagation schema proofs.

\end{enumerate}

\subsubsection{The Propagation Schema}

We present the following schematic formula whose constructive proofs capture
the methods of generating support for a particular property $P$.

\begin{definition}\name{def:support-framework}{Propagation Schema}
Given a schema $X$, a signature $\sigma$ and a p-admissible
property $P$, constructive proofs of the following statement yield a propagator
for $P$.
\[\begin{array}{l}
\forall {}S\in\Support{X}{\sigma}{P}. \\
  \abit{.5} \forall{}   \sigma_1 \sqsubseteq \sigma .\: \nonempty{\sigma_1} \Rightarrow \\
  \abit{1} S \not\in \Support{X}{\sigma_1}{P} \Rightarrow \\
  \abit{1.5} \mathrm{findNewSupport}(X, P, \sigma_1) \\
  \abit{1.5} \vee \mathrm{noNewSupport}(X, P, \sigma_1)
\end{array}\]
When an existing support \(S\) has been lost in a signature
\(\sigma_1\sqsubseteq \sigma\), a new support and a new signature \(\sigma_2 \sqsubseteq \sigma_1\) are found
in findNewSupport. Otherwise, noNewSupport states that there is no new support to be found. 

\[\abit{-1}\begin{array}{l}
\mathrm{findNewSupport}(X, P, \sigma_1) \definedAs \\
  \abit{0.5} (\exists{} \sigma_2\sqsubseteq\sigma_1.\: \nonempty{\sigma_2}\: \wedge \\
  \abit{1} \exists{} S'\in\Support{X}{\sigma_2}{P}. \\
  \abit{1.5} \forall{} \sigma_3.\, \sigma_2 \sqsubset \sigma_3 \sqsubseteq \sigma_1 \Rightarrow  \Support{X}{\sigma_3}{P} = \emptyset)\\
\end{array}\]

\[\begin{array}{l}
\mathrm{noNewSupport}(X, P, \sigma_1) \definedAs \\
\abit{1}\forall \sigma_2 \sqsubseteq \sigma_1.\:\nonempty{\sigma_2} \Rightarrow \\
\abit{2}\Support{X}{\sigma_2}{P} = \emptyset
\end{array}\]
\end{definition}

We are interested in constructive proofs\footnote{There is a classical
proof of propagator schema that is independent of the property $P$ and carries
no interesting computational content.} of the propagator schema when $P$ is
instantatied to individual support properties.  

Given an admissible support property $P$, a constructive proof of the
propagator schema yields a function that takes as input a set $S$, evidence
that $S\in\Support{X}{\sigma}{X}$, a signature $\sigma_1$ and evidence that
$\sigma_1\sqsubseteq\sigma$, evidence that $S\not\in{}\Support{X}{\sigma_1}{P}$
and returns one of two items:
\begin{description}
\item{i.)} a new signature $\sigma_2$, together with evidence that
$\sigma_2\sqsubseteq\sigma_1$, a set of tuples $S'$ and evidence that
$S'\in\Support{X}{\sigma_2}{P}$ and evidence that $\sigma_2$ is maximal.
\item{ii.)} Evidence that there is no support for $P$ in $\sigma_1$ or for any
smaller signature.
\end{description}

\begin{lemma}\name{lem:s-nonempty}{non-empty in propagation schema}
In the propagation schema, if we assume the antecedent $S\notin
\Support{X}{\sigma_1}{P}$ for $\sigma_1 \sqsubseteq \sigma$ then $S\ne\emptyset$.
\end{lemma}
\begin{proof}
By p-admissibility of $P$, if $\emptyset\in \Support{X}{\sigma}{P}$ then for all
$\sigma_1\sqsubseteq \sigma$, $\emptyset\in \Support{X}{\sigma_1}{P}$.
\end{proof}

\section{Generating Propagators}

In this section we present two case studies of applying our methodology.

\subsection{A Propagator for the Element Constraint}\label{sec:element-example}

The element constraint is widely useful in specifying a large class of
constraint problems.  It has the form ${\tt{element}}(X,y,z)$ where $X$ is a
vector of variables and $y$ and $z$ are variables.  The meaning of the element
constraint is the set of all coherent tuples on the schema
$\V{X\cdot{}y\cdot{}z}$ of the following form.
\[\tau = \V{v_1,\cdots{},v_{i-1},j,v_{i+1},\cdots,v_k,i,j}\]  
Thus, $\tau[k+1]=i$ indexes $\V{v_1,\cdots,v_k}$ and $\tau[k+2]=\tau[i]$.

\begin{definition}\name{def:element-semantics}{Element Semantics}
\[
\begin{array}{l}
\sem{{\tt{element}}(X,y,z)}{\sigma} = \V{\V{X\cdot y\cdot z},R} \\
\abit{1}\mathrm{where} \\ 
\abit{1} R = \{\tau{}\in \tuple{\V{X\cdot y\cdot z}}{\sigma} \alt\\
\abit{2} k = |X| \wedge \tau[k+1]\in\{1..k\} \\
\abit{3} \wedge \tau{}[k+2]=\tau{}[\tau{}[k+1]] \}\\
\end{array}
\]
\end{definition}

The element constraint is widely used because it represents the very basic 
operation of indexing a vector~\cite{VanHentenryckSimonisEa:92}.  For example, Gent et al.~model Langford's 
number problem and quasigroup table generation problems using 
element~\cite{Gent_Jefferson_Miguel06}. 

In \cite[pp. 188]{Gent_Jefferson_Miguel06} three properties to establish GAC
for the element constraint are characterized.  We restate theorem 1 from that
paper here:

\begin{theorem}\name{thm:GJM-06}{Theorem 1 of reference {\rm{\cite{Gent_Jefferson_Miguel06}}}.}
Given an element constraint of the form ${\tt{Element}}(X,y,z)$, domains given
by a signature $\sigma$ are Generalized Arc Consistent if and only if all of
the following hold.
\[\begin{array}{lr}
\forall{}i\in\sigma(y).\;\, \sigma(y)=\{i\} \Rightarrow \sigma(X[i]) \subseteq \sigma(z) & (1) \\
\forall{}i \in \sigma(y).\;\, \sigma(X[i]) \cap \sigma(z) \not=\emptyset & (2) \\
\sigma(z) \subseteq \displaystyle{\bigcup_{i\in{}\sigma(y)}\sigma(X[i])} & (3) \\
\end{array}\]
\end{theorem}

\comment{
Later on the same page, they present three conditions for support for a tuple
to provide support.  Given an element constraint of the form
${\tt{Element}}(X,y,z)$, a tuple of the form $\V{\tau,i,j}$, is
\begin{description}
\item{\bf Support for $\tau[i]=j$:}  Either $|\sigma(y)|>1$ or $j\in{}\sigma(z)$.
\item{\bf Support for $\V{y,i}$:}  $j\in{}\sigma(X[i])$ and $j\in{}\sigma(z)$.
\item{\bf Support for $\V{z,j}$:} $i\in{}\sigma(y)$ and  $j\in{}\sigma(X[i])$.
\end{description}
}

\subsubsection{Support Properties}\label{sec:element-support-properties}

Each of the three properties above can be characterized as properties of their
generalized supports.

\begin{definition}\name{def:element-support-properties}{Element Support Properties}
Given a schema $X$ and variables $y$ and $z$ and a signature $\sigma$ there are
three properties corresponding to three propagators for establishing GAC for
the element constraint ${\tt{Element}}(X,y,z)$. Let $k$ be $|X|$, then $k+1$ is
the index of $y$ and $k+2$ is the index of $z$ in the schema
$(X\cdot{}y\cdot{}z)$. 
\[\begin{array}{l}
P_1[\sigma](S) \definedAs \\
\abit{1} ( \exists i,j : \sigma(y) . \\
\abit{2}i\ne j\: \wedge\: \V{k+1,i} \in S \wedge \V{k+1,j} \in S ) \\
\abit{1} \vee  \;\; \forall{}i:\sigma(y).\;\forall{}a:\sigma(X[i]). \;\V{k+2, a}\in S \\
P_2[\sigma](S)\definedAs \forall{}i:\sigma(y).\;\exists{}a:\sigma(z).\\
\abit{2}\V{i,a}\in S \wedge \V{k+2,a}\in S \\
P_3[\sigma](S)\definedAs \forall{}a:\sigma(z).\; \exists{}i:\sigma(y).\\
\abit{2}\V{i,a}\in S \wedge \V{k+1,i}\in S
\end{array} \]
\end{definition}

Note that for property $P_1$, the first disjunct is true iff the domain of the
index variable $y$ has more than one element, $|\sigma(y)|>1$.  Support for
this disjunct is a pair of literals $\V{k+1,i}$ and $\V{k+1,j}$ where
$i,j\in{}\sigma(y),\; i\ne j$.\footnote{This specification corresponds to a set 
    of dynamic literal triggers \cite{Gent_Jefferson_Miguel06}. Ideally a static
    assignment trigger would be used for $P_1$, which would trigger the
    propagator
    when $y$ is assigned. However, assignment triggers are outside the scope of
    this paper.}
Logically, $(\exists i,j : \sigma(y) . \abit{1}
i\ne j)$ is equivalent, but for our purposes we must provide p-admissible
support.   Once the domain of
the index variable is a singleton ($\sigma(y)=\{i\}$), the second disjunct of
$P_1$ may be satisfied.  This disjunct is supported by a set of
$|\sigma(X[i])|$ literals of the form $\pair{k+2,a}$, one literal for each
$a\in \sigma(X[i])$. This is evidence for $\sigma(X[i])\subseteq\sigma(z)$
since $k+2$ is the index of $z$ in the schema $(X\cdot{}y\cdot{}z)$.

Property $P_2$ is supported iff $\sigma(X[i])\cap\sigma(z)$ is non-empty for
every $i\in\sigma(y)$. The support is $2m$ literals where $m=|\sigma(y)|$, two
for each $i\in{}\sigma(y)$.  These have the form $\V{i,a}$ and $\V{k+2,a}$
where $a$ is some value in $\sigma(z)$. If there is no support, then
$\sigma(X[i])\cap\sigma(z)=\emptyset$.

Property $P_3$ is supported iff $\sigma(z)\subseteq \bigcup_{i\in \sigma(y)} \sigma(X[i])$. 
The support is a set of $2m$ literals where $m=|\sigma(z)|$,
two for each $a\in{}\sigma(z)$.  The literals have the form $\V{i,a}$ and
$\V{k+1,i}$ where $i$ is some value in $\sigma(y)$. If there is no support then
for some $a\in\sigma(z)$, for all $i$ $a\not\in\sigma(X[i])$.

It is easy to prove  that the three properties act as intended:

\begin{theorem}
Given a signature $\sigma$, we have: 
\begin{itemize} 
\item (1) is true if and only if $\exists S: P_1[\sigma](S)$
\item (2) is true if and only if $\exists S: P_2[\sigma](S)$
\item (3) is true if and only if $\exists S: P_3[\sigma](S)$
\end{itemize}
\end{theorem}
\begin{proof}
The if directions are all easy.   For (1), if the first disjunct of $P_1$ is satisfied then $|\sigma(y) | > 1$ so (1) is vacuous.  If the second disjunct is satisfied, it ensures that $\sigma(X[i]) \subseteq \sigma(z)$.  If $P_2(S)$ is true then, for each element of the domain of the index variable $y$, there is a value $a \in \sigma(X[i]) \cap \sigma(z)$, establishing (2).   If $P_3(S)$ is true 
then, for any value $a$ in $\sigma(z)$ there is a value $i$ of the index variable with $a \in \sigma(X[i])$, proving that (3) holds.

For Only if, first suppose that (1) is true.   If  $|\sigma(y) | > 1$ then we can find $i,j$ to satisfy the first disjunct of $P_1$, and 
set $S = \{\pair{k+1,i},\pair{k+1,j}\}$.  Otherwise, we have $\sigma(y) = \{i\}$ and $\sigma(X[i]) \subseteq \sigma(z)$.  
We can thus set $S = \{\pair{k+2,a}|a \in \sigma(X[i])\}$. 

Suppose (2) is true.   We have $\sigma(X[i]) \cap \sigma(z) \ne \emptyset$ for each $i \in \sigma(y)$.    So for each $i$ there is thus some 
$a_i$ with $a_i \in \sigma(X[i]) \cap \sigma(z)$.   We can thus set $S = \{ \pair{i,a_i}, \pair{k+2,a_i}| i \in \sigma(y)\}$.

Suppose (3) is true.   Since $\sigma(z) \subseteq \displaystyle{\bigcup_{i\in{}\sigma(y)}\sigma(X[i])} $, we have for each $a \in \sigma(z)$ some $i_a$ such that $i_a \in \sigma(y)$ and $a \in \sigma(X[i_a])$.   We can thus set 
$S = \{ \pair{i_a,a}, \pair{k+1,i_a} \mid a \in \sigma(z)\}$.
\end{proof}

\subsubsection{P-Admissibility and Backtrack Stability} \label{sec:element-p-admiss}

Following our methodology, we first prove that properties $P_1$, $P_2$ and $P_3$ are p-admissible.

\begin{lemma}\name{lemma:P1-is-admissible}{$P_1$ is p-admissible}
\[P\!-\!admissible(P_1)\hfill\]
\end{lemma}
\begin{proof}
We case split on the disjuncts of $P_1$. The first disjunct requires distinct
values $i,j \in \sigma(y)$. Assuming $S\subseteq \tuple{\V{X\cdot y\cdot
    z}}{\sigma'}$, $i,j \in \sigma'(y)$ because the two necessary literals are
in $S$, therefore $P_1[\sigma'](S)$ holds.

For the second disjunct of $P_1$, since $\sigma' \sqsubseteq \sigma$ we can see
that $\sigma'(y) \subseteq \sigma(y)$ and $\forall i.\:\sigma'(X[i])\subseteq
\sigma(X[i])$, therefore all necessary literals are present in $S$ and
$P_1[\sigma'](S)$ holds.
\end{proof}

\begin{lemma}\name{lemma:P2-is-admissible}{$P_2$ is p-admissible}
\[P\!-\!admissible(P_2)\hfill\]
\end{lemma}
\begin{proof}
Since $\sigma' \sqsubseteq \sigma$, $\sigma'(y)\subseteq \sigma(y)$ therefore
there are fewer (or the same) values of $i$ to consider under
$\sigma'$. Assuming $S\subseteq \tuple{\V{X\cdot y\cdot z}}{\sigma'}$, for each
$i$, $\V{k+2,a}\in S$ therefore $a\in \sigma'(z)$ and $P_2[\sigma'](S)$ holds.
\end{proof}

\begin{lemma}\name{lemma:P3-is-admissible}{$P_3$ is p-admissible}
\[P\!-\!admissible(P_3)\hfill\]
\end{lemma}
\begin{proof}
The proof is the same as above, with $z$ and $y$ exchanged, $i$ and $a$ exchanged, and $k+1$ substituted for $k+2$.
\end{proof}

$P_1$, $P_2$ and $P_3$ are not backtrack stable according to Def.~\ref{def:backtrack-stable}. 
However, $P_2$ and $P_3$ can be straightforwardly reformulated to be backtrack stable: the universal quantifier is expanded to a conjunction using the initial signature, then each conjunct is made into an individual property, subscripted by $i$ or $a$ respectively. For example, $P_2$ is transformed as follows. 

\[
\begin{array}{ll}
P_{2,i}[\sigma](S)\definedAs i \in \sigma(y)\;\Rightarrow  &(\exists{}a:\sigma(z).\; \V{i,a}\in S \\ 
& \wedge \: \V{k+2,a}\in S) \\
\end{array}
\] 

Each of these smaller properties then requires two literals as support, or (if $i \notin \sigma(y)$) the empty set, and they are backtrack stable.  
$P_1$ can be reformulated to be backtrack stable, by expanding out the universal quantifiers in the same way as for $P_2$. $P_1$ would be subscripted by $i$ and $a$, $\forall i:\sigma(y)$ replaced with $i\in \sigma(y) \Rightarrow$, and the same for $\forall a:\sigma(X[i])$. These reformulations give a large set of properties, so for the sake of simplicity we use the original $P_1$, $P_2$ and $P_3$.

\subsubsection{Proofs of the Propagation Schema}

Now that we have established p-admissibility for each of $P_1$, $P_2$ and $P_3$
we prove the instances of the propagator schema for each of them.

\begin{theorem}[$P_1$ Support Generation]
\label{thm:p1-support-generation}
We consider $P_1$ on constraint ${\tt{Element}}(X,y,z)$. 
We claim that Def.~\ref{def:support-framework} (propagation schema) holds for $P_1$. 
\end{theorem}
\begin{proof}
Let $C$ be an element constraint of the form ${\tt{Element}}(X,y,z)$ where
$|X|=k$ and let $\sigma$ and $\sigma_1$ be signatures mapping the variables in
$X.y.z$ to their respective domains.  We claim the following:
\[\begin{array}{l}
\forall {}S\in\Support{X}{\sigma}{P}. \\
  \abit{.5} \forall{}   \sigma_1 \sqsubseteq \sigma .\: \nonempty{\sigma_1} \Rightarrow \\
  \abit{1} S \not\in \Support{X}{\sigma_1}{P} \Rightarrow \\
  \abit{1.5} \mathrm{findNewSupport}(X, P, \sigma_1) \\
  \abit{1.5} \vee \mathrm{noNewSupport}(X, P, \sigma_1)
\end{array}\]

\[\abit{-1}\begin{array}{l}
\mathrm{findNewSupport}(X, P, \sigma_1) \definedAs \\
  \abit{0.5} (\exists{} \sigma_2\sqsubseteq\sigma_1.\: \nonempty{\sigma_2}\: \wedge \\
  \abit{1} \exists{} S'\in\Support{X}{\sigma_2}{P}. \\
  \abit{1.5} \forall{} \sigma_3.\, \sigma_2 \sqsubset \sigma_3 \sqsubseteq \sigma_1 \Rightarrow  \Support{X}{\sigma_3}{P} = \emptyset)\\
\end{array}\]

\[\begin{array}{l}
\mathrm{noNewSupport}(X, P, \sigma_1) \definedAs \\
\abit{1}\forall \sigma_2 \sqsubseteq \sigma_1.\:\nonempty{\sigma_2} \Rightarrow \\
\abit{2}\Support{X}{\sigma_2}{P} = \emptyset
\end{array}\]

The proof consists of constructing $\sigma_2$ and $S'$ for all cases, given $\sigma_1$. When $\sigma_2 \sqsubset \sigma_1$, we also prove that $\sigma_2$ is maximal (\ie~there exists no $\sigma_3$). 
\[
\abit{-2}\begin{array}{l}
|\sigma_1(y)|>1 \Rightarrow  \\
\abit{1}S'=\{ \V{k+1, min(\sigma_1(y))}, \V{k+1, max(\sigma_1(y))} \} \\
\abit{2}\wedge \sigma_2=\sigma_1 \\
\sigma_1(y)=\{i\} \Rightarrow \\
\abit{1}\sigma_2(X[i])=\sigma_1(z)\cap \sigma_1(X[i])  \\
\abit{1}\wedge\: (\forall x \in \V{X\cdot y\cdot z}.\:x\ne X[i]\;\Rightarrow \; \sigma_2(x)=\sigma_1(x))\: \\
\abit{1}\wedge\: S'=\bigcup_{b\in \sigma_2(X[a])}\{ \V{k+2, b} \}
\end{array}
\]
For the second case above, it remains to be shown that $\sigma_2$ is nonempty and maximal. We prove that $\sigma_2$ is maximal. For all values $b\in \sigma_2(X[a])$, a supporting literal $\V{z, b}$ is required in $S'$. Therefore, $P_1$ implies that $\sigma_2(X[a])\subseteq \sigma_2(z)$, hence $\sigma_2(X[a])=\sigma_1(z)\cap \sigma_1(X[a])$ is maximal.
For all other variables $w$, $\sigma_2(w)=\sigma_1(w)$, therefore $\sigma_2$ is maximal under $\sqsubseteq$.

If $\sigma_2(X[i])=\emptyset$ (\ie~$\sigma_1(z)\cap \sigma_1(X[i]) =\emptyset$), $\sigma_2$ is empty. Since $\sigma_2$ is the maximal one which satisfies $P_1$, the second disjunct (noNewSupport) of the consequent of the schema holds.
\end{proof}

\begin{theorem}[$P_2$ Support Generation]
We consider $P_2$ on constraint ${\tt{Element}}(X,y,z)$. 
We claim that Def.~\ref{def:support-framework} (propagation schema) holds for $P_2$.
\end{theorem}
\begin{proof}
Let $k=|X|$, and $\sigma_1$ and $\sigma_2$ be signatures mapping the variables in
$X.y.z$ to their respective domains. 
The proof is by constructing $\sigma_2$ and $S'$ to satisfy the first disjunct of the consequent of the schema. 
\[
\abit{-1}\begin{array}{l}
\sigma_2(y)= \{i\in \sigma_1(y) \mid \exists a\in \sigma_1(z).\: a \in \sigma_1(X[i]) \} \\
\forall x \in \V{X.z}\:\sigma_2(x)=  \sigma_1(x) \\
S'= \bigcup_{i\in\sigma_2(y)}\{ \V{i, a}, \V{k+2, a} \} \\
\end{array}
\]
$\sigma_2$ is maximal: the constructed $\sigma_2$ is identical to $\sigma_1$ except for the set $\sigma_2(y)$. For each value $i$ of $\sigma_2(y)$, $P_2$ requires that there exists a value $a$ in the domains of $X[i]$ and $z$. $\sigma_2(y)$ is the maximal subset of $\sigma_1(y)$ which satisfies this condition, therefore $\sigma_2$ is maximal under $\sqsubseteq$. 

If $\sigma_2$ is empty, then (since $\sigma_2$ is maximal) the second disjunct of the consequent of the schema holds. 
\end{proof}

\begin{theorem}[$P_3$ Support Generation]
We consider $P_3$ on constraint ${\tt{Element}}(X,y,z)$. 
We claim that Def.~\ref{def:support-framework} (propagation schema) holds for $P_3$.
\end{theorem}
\begin{proof}
Let $k=|X|$, and $\sigma_1$ and $\sigma_2$ be signatures mapping the variables in
$X.y.z$ to their respective domains. 
The proof is by constructing $\sigma_2$ and $S'$ to satisfy the first disjunct of the consequent of the schema.
\[
\abit{-1}\begin{array}{l}
\sigma_2(z)= \{a\in \sigma_1(z) \mid \exists i\in \sigma_1(y).\: a \in \sigma_1(X[i]) \} \\
\forall x \in X.y.\: \sigma_2(x)= \sigma_1(x) \\
S'= \bigcup_{a\in \sigma_2(z)}\{ \V{i, a}, \V{k+1, i} \} \\
\end{array}
\]
The constructed $\sigma_2$ is identical to $\sigma_1$ except for the set $\sigma_2(z)$. For each value $a$ of $\sigma_2(z)$, $P_3$ requires that there exists an index $i$ such that $a\in \sigma_2(X[i])$ and $i\in \sigma_2(y)$. $\sigma_2(z)$ is the maximal subset of $\sigma_1(y)$ which satisfies this condition, therefore $\sigma_2$ is maximal under $\sqsubseteq$.

If $\sigma_2$ is empty, then (since $\sigma_2$ is maximal) the second disjunct of the consequent of the schema holds. 
\end{proof}

\subsubsection{Soundness and Completeness}

Now we prove that the conjunction of the element support properties
(Def.~\ref{def:element-support-properties}) is sound and complete using the
semantics of element (Def.~\ref{def:element-semantics}).  We will write $P_e$
for the set $\{P_1,P_2,P_3\}$.

\begin{lemma}\name{lemma:element-sound}{$P_{e}$ is sound}
\[
\abit{-1}\begin{array}{l}
\forall{}\sigma.\;\mathrm{singleton}(\sigma) \Rightarrow\\ 
\abit{.5}(\Support{X}{\sigma}{P_{e}} \Rightarrow \sem{{\tt{element}}(X,y,z)}{\sigma} \ne \emptyset)
\end{array}
\]
\end{lemma}
\begin{proof}
Let $\sigma$ be an arbitrary singleton signature. Since $\sigma$ is a singleton
it encodes a single tuple (say $\tau$).  Assume $\Support{X}{\sigma}{P_{e}}$
holds. That is, supports for $P_1[\sigma]$, $P_2[\sigma]$ and $P_3[\sigma]$ are
non empty.
\comment{
\[\begin{array}{ll}
i.) & P_1[\sigma](S) \definedAs \\
& \abit{1} ( \exists i,j : \sigma(y) . \abit{1} i\ne j\: \wedge\: \V{k+1,i} \in S \wedge \V{k+1,j} \in S ) \\
& \abit{1} \vee  \;\; \forall{}i:\sigma(y).\;\forall{}a:\sigma(X[i]). \;\V{k+2, a}\in S \\
ii.) & P_2[\sigma](S)\definedAs \forall{}i:\sigma(y).\;\exists{}a:\sigma(z).\; \V{i,a}\in S \wedge \V{k+2,a}\in S \\
iii.) & P_3[\sigma](S)\definedAs \forall{}a:\sigma(z).\; \exists{}i:\sigma(y).\; \V{i,a}\in S \wedge \V{k+1,i}\in S
\end{array}\]
} Now, consider $P_1$. Since $|\sigma(y)|=1$ we know the first disjunct can not
hold and so we must have support for the second. Since $\sigma(y)\ne \emptyset$ we
know that there is a single tuple supporting the second disjunct of $P_1$ and
since $|\sigma(X[i])|=1$, to support $P_1$, $\tau$ must have the form
$\V{x_1,\cdots,x_{i-1},a,x_{i+1},\cdots,x_k,i,a}$. This same tuple supports
$P_2$ and $P_3$.  This tuple is clearly in
$\sem{{\tt{element}}(X,y,z)}{\sigma}$ and so soundness holds.
\end{proof}

\begin{theorem}\name{thm:P1-complete}{$\{P_1\}$ complete}
\[
\abit{-1}\begin{array}{l}
\forall{}\sigma.\;\; \sem{{\tt{element(X,y,z)}}}{\sigma} \ne \emptyset \;\Rightarrow \\
\abit{1}\exists{}\sigma'\sqsubseteq{}\sigma.\\
\abit{2}\sem{{\tt{element(X,y,z)}}}{\sigma'}=\sem{{\tt{element(X,y,z)}}}{\sigma}\\
\abit{2} \wedge \: \Support{Y}{\sigma'}{\{P_1\}}
\end{array}
\]
\end{theorem}
\begin{proof}
Assume $\sem{{\tt{element}}(X,y,z)}{\sigma}\ne \emptyset$ for arbitrary $\sigma$.
If $\Support{Y}{\sigma}{P_1} \ne{}\emptyset$ then the theorem is trivially true, so
we assume that $\Support{Y}{\sigma}{P_1}=\emptyset$ and construct a signature
$\sigma'$ that does not eliminate any solutions from the constraint and in which
$P_1$ has support.

The first disjunct of $P_1$ is supported whenever $|\sigma(y)|>1$ and so if
$P_1$ is not supported $\sigma(y)=\{i\}$ or $\sigma(y)=\emptyset$; by assumption no
domain of $\sigma$ is empty and so $\sigma(y)=\{i\}$.  To falsify the second
disjunct of $P_1$ when $\sigma(y)=\{i\}$, there must be some $a\in\sigma(X[i])$
such that the literal $\V{k+2,a}$ can not be supported.  This happens for any
$a\in\sigma(X[i])$ where $a\not\in\sigma(z)$.  Let $\sigma_1$ be a
signature that is just like $\sigma$ except that \[\sigma_1(z) =
\sigma(z)\cap\sigma(X[i])\] Since the constraint is non-empty the
intersection is non-empty.  The second disjunct of $P_1$ supports this new
signature so it supports $P_1$.  Clearly $\sigma_1\sqsubseteq\sigma$ and so it
only remains to show that the meaning of the constraint does not change under
the new signature.  It is enough to show that
\[\sem{{\tt{element}}(X,y,z)}{\sigma}\subseteq\sem{{\tt{element}}(X,y,z)}{\sigma_1}\]
Assume $\tau\in\sem{{\tt{element}}(X,y,z)}{\sigma}$.  Then
$\tau\in\tuple{\V{X\cdot y\cdot z}}{\sigma}$ is coherent and  is of the form
\[\tau=\V{x_1,\cdots,x_{i-1},a,x_{i+1},\cdots,x_k,i,a}\] 
Since $\tau$ is an $\tuple{\V{X\cdot y\cdot z}}{\sigma}$, we know $\tau[j]\in\sigma(X[j])$ for
all $j\in\{1..k+2\}$.  To construct $\sigma_1$ we simply eliminated elements
$b\in\sigma(z)$ such that $b\not\in\sigma(X[i])$ so since
$a\in\sigma(X[i])$, $a\in\sigma_1(X[i])$ and $a\in\sigma_1(z)$ and so
$\tau\in\sem{{\tt{element}}(X,y,z)}{\sigma_1}$.
\end{proof}

\begin{theorem}\name{thm:P2-complete}{$\{P_2\}$ complete}
\[
\abit{-1}\begin{array}{l}
\forall{}\sigma.\;\; \sem{{\tt{element(X,y,z)}}}{\sigma} \ne \emptyset \;\Rightarrow \\
\abit{1}\exists{}\sigma'\sqsubseteq{}\sigma.\\
\abit{2}\sem{{\tt{element(X,y,z)}}}{\sigma'}=\sem{{\tt{element(X,y,z)}}}{\sigma}\\
\abit{2} \wedge \: \Support{Y}{\sigma'}{\{P_2\}}
\end{array}
\]
\end{theorem}
\begin{proof}
Note that if $P_2[\sigma]$ is unsupported then $\sigma(X[i])\cap\sigma(z) =
\emptyset$.  But since we assume that $\sem{{\tt{element(X,y,z)}}}{\sigma} \ne \emptyset$,
this is impossible and so $P_2[\sigma]$ must be supported and completeness
trivially holds.
\end{proof}

\begin{theorem}\name{thm:P3-complete}{$\{P_3\}$ complete}
\end{theorem}
\begin{proof}
If there is no support for $P_3[\sigma]$ then
\[\exists{}a\in\sigma(z).\:\forall{}i\in\sigma(y).\:a\not\in\sigma(X[i])\]
Just let $\sigma'$ be the same as $\sigma$ except that we remove all such
elements from the domain of $z$ in $\sigma'$.
\[\sigma'(z) = \sigma(z)\cap \bigcup_{i\in{\sigma(y)}}\sigma(X[i])\]
Clearly $\sigma'(z)\subset\sigma(z)$.  The elements that have been removed
could not be included in a solution of $\sem{{\tt{element(X,y,z)}}}{\sigma}$
and so we have lost no answers. Thus, we have shown $P_3$ is complete.
\end{proof}

\begin{corollary}\name{corollary:element-complete}{$P_{e}$ is complete}
\end{corollary}
\begin{proof}
The completeness of $P_e$ follows from local completeness
(Thm.~\ref{thm:local-completeness}) and the completeness of $P_1$, $P_2$ and
$P_3$.
\end{proof}

\subsubsection{Discussion}

The propagators derived here to enforce GAC on the element constraint are not identical to those presented by Gent et al.~\cite{Gent_Jefferson_Miguel06}. However they do follow the same general scheme. The main difference is that the propagators here use dynamic literal triggers in place of watched literals and a static assignment trigger. The concept of generalized support has allowed us to create these propagators within one formal framework. 

\subsection{New Watched Literal Propagators for Occurrence Constraints}

The two constraints {\tt occurrenceleq}$(X, a, c)$ and {\tt occurrencegeq}$(X, a, c)$ (very similar to \texttt{atmost} and \texttt{atleast}) 
restrict the number of occurrences of a value in a vector of variables. If $\mathrm{occ}(X, a)$ is the occurrences of value $a$ in $X$, {\tt occurrenceleq} states that $\mathrm{occ}(X, a)\le c$ and {\tt occurrencegeq} states that $\mathrm{occ}(X, a)\ge c$. 

Occurrence constraints arise in many problems. For example, in a round-robin tournament schedule, it may be required that no team plays more than twice at each stadium~\cite{csplib-roundrobin}, represented by occurrenceleq constraints. In car sequencing (car factory scheduling), occurrence constraints may be used to avoid placing too much demand on a work-station~\cite{csplib-carseq}. 

First we present the formal semantics of occurrenceleq and occurrencegeq, followed by support properties for the two constraints.

\begin{definition}\name{def:occurrenceleq-semantics}{Occurrenceleq Semantics}
\[
\abit{-1}\begin{array}{l}
\sem{{\tt{occurrenceleq}}(X,a,c)}{\sigma} = \V{X,R_X} \abit{0.5} {\rm where\ } \\
\abit{1} R_X = \{ \: \tau{}\in \tuple{X}{\sigma} \mid \\
\abit{8} \left| \left\{ i \mid \tau[i] = a \right\}\right| \: \le\: c \quad \}
\end{array}
\]
\end{definition}

\begin{definition}\name{def:occurrencegeq-semantics}{Occurrencegeq Semantics}
\[
\abit{-1}\begin{array}{l}
\sem{{\tt{occurrencegeq}}(X,a,c)}{\sigma} = \V{X,R_X} \abit{0.5} {\rm where\ } \\
\abit{1} R_X = \{ \: \tau{}\in \tuple{X}{\sigma} \mid \\
\abit{8} \left| \left\{ i \mid \tau[i] = a \right\}\right| \: \ge\: c \quad \}
\end{array}
\]
\end{definition}

\subsubsection{Support Properties}

\begin{definition}\name{def:occurrence-support-properties}{Occurrence Support Properties}
Given a schema $X$, value $a$ and occurrence count $c$, $P_l$ is the support property for the {\tt occurrenceleq} constraint, and similarly $P_g$ is the property for {\tt occurrencegeq}.
\[\begin{array}{rl}
P_l[\sigma](S) \definedAs & ( \exists I \subseteq \{1\ldots |X|\}.\\
& |I| = (|X|-c+1) \wedge \\
& \forall i \in I. \:\exists b\ne a.\: \langle i,b \rangle \in S )\\
& \vee \\
& (\exists I \subseteq \{1\ldots |X|\}.\\ 
& |I| = (|X|-c) \: \wedge \\
& \forall i\in I.\: a\notin \sigma(X[i]))\\
\end{array}\]
\[\begin{array}{rl}
P_g[\sigma](S) \definedAs & ( \exists I \subseteq \{1\ldots |X|\}.\\
& |I| = (c+1) \: \wedge \\
& \forall i \in I.\: \langle i,a \rangle \in S )\\
& \vee \\
& (\exists I \subseteq \{1\ldots |X|\}.\\
& |I| = c \:\: \wedge \\
& \forall i\in I.\: \nexists b\in \sigma(X[i]) . \: b\ne a)\\
\end{array}\]
\end{definition}

$P_g$ is slightly simpler, so we consider it first. There are two forms of support which can satisfy $P_g$, corresponding to the two disjuncts. The first disjunct can be satisfied if $c+1$ variables have $a$ in their domain, by a support set which contains $c+1$ literals mapping distinct variables to $a$. The second disjunct is satisfied if $c$ variables are {\em set} to $a$. In this case, $S$ may be empty. 

When it is no longer possible to satisfy the first disjunct, a corresponding propagator must narrow the domains to satisfy the second disjunct, by setting $c$ variables to $a$. At this point, the constraint is trivially satisfied so $S$ may be empty.

$P_l$ is very similar, and essentially works in the same way except that it requires $|X|-c$ non-occurrences of $a$ rather than $c$ occurrences of $a$.

\subsubsection{P-Admissibility and Backtrack Stability}

We now prove that both properties meet the p-admissibility requirement.

\begin{theorem}\name{thm:pl-admiss}{$P_l$ P-Admissible}
$P_l$ is p-admissible according to Def.~\ref{def:p-admissible}.
\end{theorem}
\begin{proof}
We case split on the disjuncts of $P_l$. The first disjunct does not refer to
$\sigma'$, and (since $S$ has not changed) it remains true.
The second disjunct is satisfied by $S=\emptyset$ only when the constraint is a tautology. Since  $a\notin \sigma(X[i])$ and $\sigma' \sqsubseteq \sigma$,
then $a \notin \sigma'(X[i])$ and the property remains true. 
\end{proof}

\begin{theorem}\name{thm:pg-admiss}{$P_g$ P-Admissible}
$P_g$ is p-admissible according to Def.~\ref{def:p-admissible}.
\end{theorem}
\begin{proof}
We case split on the disjuncts of $P_g$. The first disjunct does not refer to
$\sigma'$, and (since $S$ has not changed) it remains true.
The second disjunct is satisfied by $S=\emptyset$ only when the constraint is a tautology. Since  $\sigma(X[i]) \subseteq \{a\}$ and $\sigma' \sqsubseteq \sigma$,
then $\sigma'(X[i]) \subseteq \{a\}$ and the property remains true. 
\end{proof}

In order for the two propagators to make use of watched literals, we must prove that both properties are backtrack stable. The watched literals representing a support are not backtracked, so a support must remain a support as search backtracks (and the domains are widened).  

\begin{theorem}\name{thm:occ-bt-stable}{Occurrence Backtrack Stable}
The two occurrence support properties are backtrack stable according to Def.~\ref{def:backtrack-stable}. 
\end{theorem}
\begin{proof}
For both properties, the second disjunct is irrelevant because it is satisfied by $S=\emptyset$ only when the constraint is a tautology. The support $\emptyset$ is not required to be backtrack stable.
In both properties the first disjunct requires a fixed number ($|X|-c+1$ or $c+1$) of literals to be in $S$ (with variable indices $I$). It is clear that for any $\sigma'$ where $\sigma\sqsubseteq \sigma'$, the same $I$ may be used to discharge the existential, and $S$ will be valid w.r.t $\sigma'$. 
\end{proof}

\subsubsection{Proofs of the Propagation Schema}

Now we give a constructive proof of the propagation schema for $P_l$. Recall that the computational content of the proof is a propagator for $P_l$. 

\begin{theorem}[$P_l$ Support Generation]\label{thm:pl-support-generation}
We consider $P_l$ on constraint ${\tt{occurrenceleq}}(X,a,c)$. 
We claim that Def.~\ref{def:support-framework} (propagation schema) holds for $P_l$. 
\end{theorem}
\begin{proof}
Let $\sigma_1$ and $\sigma_2$ be signatures mapping the variables in
$X$ to their respective domains. 
$S$ and $\sigma_1 \sqsubseteq \sigma$ are universally quantified in the schema, therefore we use them as givens. We assume that $S\notin  \Support{X}{\sigma_1}{P_l}$ and prove the consequent by constructing $S'$ and $\sigma_2$. By lemma \ref{lem:s-nonempty}, $S\ne \emptyset$. The second disjunct of $P_l$ would be satisfied by $S=\emptyset$, therefore $S$ corresponds to the first disjunct of $P_l$. 

$S$ contains one literal for each index in $I$. At least one item in $S$ is invalid (by the antecedant). The proof proceeds by constructing $I'$ and corresponding $S'$ and $\sigma_2$ to satisfy the first disjunct of $P_l$ if possible. Otherwise, the second disjunct is satisfied by constructing $\sigma_2$ and $S'=\emptyset$. 

\[
\begin{array}{l}
I_1= \{ i \mid \V{i,b}\in S \wedge(\exists b\ne a.\; b\in \sigma_1(X[i]))\}\\
I_2= \{ i \mid i\notin I_1 \wedge (\exists b\ne a.\; b\in \sigma_1(X[i])) \}\\
I_3= I_1 \cup I_2 \\
\end{array}\]
\[
\begin{array}{l}
|I_3|>(|X|-c) \Rightarrow \\
\abit{1} (I'\subseteq I_3 \wedge |I'|=(|X|-c+1) \\
\abit{1} \wedge\: S'=\{ \V{i,b} \mid i\in I'\:   \\
\abit{3} \wedge \: b\in \sigma_1(X[i])\: \wedge\: b\ne a \} \\
\abit{1} \wedge \: \sigma_2=\sigma_1 )\\
\\
|I_3|=(|X|-c) \Rightarrow \\
\abit{1} S'=\emptyset \:\wedge  \\
\abit{1} (\forall i\notin I_3.\:\sigma_2(X[i])=\sigma_1(X[i])) \:\wedge \\
\abit{1} (\forall i\in I_3.\:\sigma_2(X[i])=\sigma_1(X[i])\setminus \{a\})\\
\end{array}
\]

$\sigma_2$ is maximal in both of the above cases: in the first case, $\sigma_2=\sigma_1$, and in the second case only the necessary values are removed to satisfy the second disjunct of $P_l$.

When $|I_3|<(|X|-c)$, $P_l$ is false and remains false for all $\sigma_2\sqsubseteq \sigma_1$ (by construction of $I_1$ and $I_2$). Hence the second disjunct of the consequent of the schema is satisfied. 
\end{proof}

The proof explicitly re-uses variable indices but not $b$ values from $S$. This fits well with Minion's watched literal implementation, which notifies the propagator once for each invalid literal in $S$. However, the proof does not require the use of watched literals, it allows many concrete implementations and may be used with any propagation-based solver.

It is straightforward to prove the propagation schema for $P_g$, based on the proof for $P_l$. 

\begin{theorem}[$P_g$ Support Generation]
We consider $P_g$ on constraint ${\tt{occurrencegeq}}(X,a,c)$. We claim that Def.~\ref{def:support-framework} (propagation schema) holds for $P_g$. 
\end{theorem}
\begin{proof}
The proof is the same as the proof of $P_l$, with $c$ substituted for $\left| X\right|-c$ in all places, and $(a\in \sigma_1(X[i]))$ substituted for $(\exists b\ne a.\; b\in \sigma_1(X[i]))$, and $\{a\}$ substituted for $\sigma_1(X[i])\setminus \{a\}$.
\end{proof}

This proof also re-uses variable indices from $S$ and thus fits well with Minion's watched literal infrastructure. 

\subsubsection{Soundness and Completeness}

Now we prove the soundness and completeness of both properties, and hence the correctness of the two propagators. 

\begin{lemma}\name{lem:occ-sound}{Occurrenceleq Sound}
\[
\begin{array}{l}
\forall{}\sigma.\;\mathrm{singleton}(\sigma) \Rightarrow\\ 
\abit{.5}(\Support{X}{\sigma}{P_l} \Rightarrow\\ \abit{1}\sem{{\tt{occurrenceleq}}(X,a,c)}{\sigma} \ne \emptyset)
\end{array}
\]
\begin{proof}
Let $\sigma$ be an arbitrary singleton signature. Since $\sigma$ is a singleton
it encodes a single tuple (say $\tau$).  Assume $\Support{X}{\sigma}{P_l}$
holds. Let $b$ be the number of occurrences of $a$ in $\tau$. 

Since $\sigma$ is singleton, the first disjunct of $P_l$ implies the second disjunct. (Assume $I$ satisfies the first disjunct. $I'\subseteq I$ where $|I'|=(|X|-c)$ is used to satisfy the second disjunct.) Therefore $\Support{X}{\sigma}{P_l}$ implies the second disjunct of $P_l$ is satisfied (by the empty support). Hence, at least $|X|-c$ elements of $\tau$ are not equal to $a$, so $b\le c$. By Def.~\ref{def:occurrenceleq-semantics}, $R_X = \{ \tau \}$ and the lemma holds.
\end{proof}
\end{lemma}

The proof that $P_g$ is sound proceeds by the same argument, with $|X|-c$ replaced with $c$, `not equal to $a$' replaced with `equal to $a$' and $\le$ replaced with $\ge$.

\begin{lemma}\name{lem:occ-complete}{Occurrenceleq Complete}
\[
\begin{array}{l}
C={\tt occurrenceleq}(X,a,c) \\
\forall{}\sigma.\;\; \sem{C}{\sigma} \ne \emptyset \;\Rightarrow \\
\abit{1}\exists{}\sigma'\sqsubseteq{}\sigma.\:\sem{C}{\sigma}\subseteq \sem{C}{\sigma'}\\
\abit{2} \wedge \: \Support{X}{\sigma'}{P_l}
\end{array}
\]
\end{lemma}
\begin{proof}
Assume $\sem{C}{\sigma} \ne \emptyset$ for arbitrary $\sigma$. If $\Support{X}{\sigma}{P_l}$ then $\sigma'=\sigma$ and completeness trivially holds. Otherwise, by the proof of the propagation schema for $P_l$, there exists a $\sigma'\sqsubset \sigma$ (named $\sigma_2$ there) such that $\Support{X}{\sigma'}{P_l}$. Since $\sigma' \ne \sigma$, $\sigma'$ is constructed in the case where $|I_3|=(|X|-c)$. $\sigma'$ is the same as $\sigma$ except for indices $I_3$, where the value $a$ is removed if present. 
For all $i\notin I_3$, $\sigma(i)=\{ a\}$ therefore corresponding elements of all tuples $\tau \in \sem{C}{\sigma}$ also equal $a$. No other element of $\tau$ can be $a$ (by Def.~\ref{def:occurrenceleq-semantics}), therefore no tuples are invalidated, $\sem{C}{\sigma'}=\sem{C}{\sigma}$ and the lemma holds.
\end{proof}

Once again, the proof that $P_g$ is complete follows the same argument. For $P_g$, $|I_3|=c$ and for all indices $i\in I_3$, $\sigma'(i)=\{a\}$. For other indices, the constructed $\sigma'$ is equal to $\sigma$ and does not contain $a$. By Def.~\ref{def:occurrencegeq-semantics}, all tuples $\tau\in \sem{C}{\sigma}$ must equal $a$ at all indices $I_3$, therefore no tuples are invalidated under $\sigma'$ and $\sem{C}{\sigma'}=\sem{C}{\sigma}$.

\subsubsection{Empirical Evaluation}

The occurrence propagators implemented in Minion 0.12 (and, to the best of our knowledge, all other solvers) use static triggers. Therefore they may be invoked when support has not been lost. By comparison, these watched literal propagators are only invoked when one of the literals in the support is lost.

We implemented the $\mathtt{occurrenceleq}(X,a,c)$ propagator described by the proof of
Theorem~\ref{thm:pl-support-generation} in Minion 0.12. The propagator re-uses
literals $\V{i,b}$ from $S$ when constructing $S'$, allowing it to leave the
corresponding watched literals in place. When a literal $\V{i,b}$ in $S$ is
invalid, the propagator scans through $X[\Rng{i}{|X|-1}]$ then $X[\Rng{0}{i-1}]$
to find a replacement literal. 
The propagator (referred to as WatchedProp) was constructed
from the proof in less than 3 hours programmer time.

We compare against the existing $\mathtt{occurrenceleq}$ propagator (StaticProp) provided in
Minion 0.12, which uses static assignment triggers (\ie~the propagator is
notified when any variable in scope becomes assigned).

We constructed a benchmark CSP as follows. We have a vector of variables $X$ where 
$|X|=100$, and initial signature $\sigma$ where 
$\forall i.\:\sigma(X[i])=\{1,2\}$. The constraints are as follows:
\[\forall i\in \{80..98\}.\:(X[i]\neq X[i+1])\] 
and 100 copies of the constraint: 
\[\mathtt{occurrenceleq}(X,1,90)\] 
The 
occurrence constraint is duplicated to allow accurate measurement of its 
efficiency.
This CSP is solved to find all solutions.

The solver branches on variables in $X$ in index order, and branches for 1 
before 2. Once variable $X[80]$ is assigned by search, the remaining variables
are assigned by propagation on the $\neq$ constraints. As search progresses,
the value of each variable in $X[\Rng{80}{99}]$ alternates between 1 and 2.

WatchedProp watches 11 literals of the form $\V{i,2}$. Early 
in the search, most of these literals will necessarily involve variables
$X[\Rng{80}{99}]$, a pathological case for WatchedProp. As search progresses,
more variables in $X[\Rng{0}{79}]$ will be assigned 2, therefore the performance
of WatchedProp should improve.

Table \ref{tab:occ-results} shows that StaticProp scales approximately 
linearly in the number of search nodes explored, but WatchedProp speeds up as 
search progresses. With a limit of 100 million nodes, WatchedProp is
more than twice as fast as StaticProp.

\begin{table*}
\begin{center}
  \begin{tabular}{r|rr}
Search node limit ($n$) & WatchedProp time (s) & StaticProp time (s) \\
    \hline
100,000     & 1.72   & 1.20   \\
1,000,000   & 12.40   & 11.54   \\
10,000,000  & 86.13   & 120.31  \\
100,000,000 & 518.81  & 1205.07  \\
    \hline
  \end{tabular}
\end{center}
\caption{
Times for the WatchedProp and StaticProp algorithms, median of 16 
runs on a dual processor Intel Xeon E5520 at 2.27GHz. 
}
\label{tab:occ-results}
\end{table*}

\subsubsection{Discussion}

We have shown that our framework can be used to create highly efficient watched literal 
propagators for occurrence constraints, and that these outperform conventional
propagators that use static triggers. 
There is no requirement for the propagators to maintain GAC. In this case we have proven that the propagators are sound and complete, the most basic requirements for correctness. The framework is entirely agnostic about whether the propagator maintains GAC, some form of bound consistency or indeed some custom consistency that is specific to the type of constraint. 

\section{Conclusions and Future Work}

This paper has made a number of contributions to the formal study of constraint solving, 
in particular of propagation in constraint solving.    
We have shown that we can define formally a notion of generalized support, which generalizes the standard notion of support 
in constraint satisfaction.   This generalization allows us to work with propagators that might not have been seen as using support. 
Since our definition is so general, we introduced the notion of ``p-admissible'' support properties.   
The definition of p-admissibility corresponds to the use of a particular kind of trigger within the constraint solver.  
Triggers are events which cause propagators to be called within the solver, and p-admissibility guarantees that any event which 
might cause support to be lost is observed by some trigger. In this paper we have focussed on a definition of p-admissibility 
corresponding to literal triggers (that are activated by deletion of a particular value from the domain of a variable). 
We have given a formal description of constraint propagation.  Given a p-admissible support property, we have defined the 
propagation schema.  A constructive proof of the propagation schema shows how a propagator can be constructed to find new support when 
support is lost.   We have given examples of this for the specific constraints ``element'', ``occurrenceleq'' and ``occurrencegeq''.

Our work on propagators is not merely a formalisation of existing standard usage in constraint programming.   
We are not aware of a definition of support as general as ours within constraints.   The notion of generalized support 
should be directly useful in constraints, enabling a much better understanding of propagation algorithms in 
the constraint community.    Our hypothesis is that almost all propagators used in constraint solvers can be seen 
as reasoning with some form of support property, even though most propagators are not currently presented as doing so.  
Once this hypothesis is confirmed, we can present propagation algorithms in a much more uniform fashion, as well as building 
constraint solvers to exploit these propagation algorithms. 
Thus our intended future work consists of two strands: first continuing the formal development we have started here, and second 
demonstrating the application of our work to the constraints community.

\section*{Acknowledgements} 
The work of the authors has been partially supported by the following UK EPSRC grants: EP/E030394/1, EP/F031114/1, EP/H004092/1\comment{dominion}, and EP/M003728/1\comment{combining constraints verification}, support for which we are very grateful.

\bibliographystyle{abbrv}
\bibliography{references}

\begin{thebibliography}{10}

\bibitem{apt-constraint-programming}
K.~R. Apt.
\newblock {\em Principles of Constraint Programming}.
\newblock Cambridge University Press, 2003.

\bibitem{apt-monfroy-auto-99}
K.~R. Apt and E.~Monfroy.
\newblock Automatic generation of constraint propagation algorithms for small
  finite domains.
\newblock In {\em Proceedings of the Fifth International Conference on
  Principles and Practice of Constraint Programming ({CP 1999})}, pages 58--72,
  1999.

\bibitem{beldiceanu-etal-linking-14}
N.~Beldiceanu, M.~Carlsson, P.~Flener, M.~A.~F. Rodr\'iguez, and J.~Pearson.
\newblock Linking prefixes and suffixes for constraints encoded using automata
  with accumulators.
\newblock In {\em Proceedings of the 20th International Conference on
  Principles and Practice of Constraint Programming ({CP 2014})}, pages
  142--157, 2014.

\bibitem{beldiceanu-etal-deriving-04}
N.~Beldiceanu, M.~Carlsson, and T.~Petit.
\newblock Deriving filtering algorithms from constraint checkers.
\newblock In {\em Proceedings of the 10th International Conference on
  Principles and Practice of Constraint Programming ({CP 2004})}, pages
  107--122, 2004.

\bibitem{BessiereHandbook}
C.~Bessi\`ere.
\newblock Constraint propagation.
\newblock In P.~V.~B. F.~Rossi and T.~Walsh, editors, {\em Handbook of
  Constraint Programming}, pages 29--83. Elsevier, 2006.

\bibitem{tractability-globals-04}
C.~Bessiere, E.~Hebrard, B.~Hnich, and T.~Walsh.
\newblock The tractability of global constraints.
\newblock In {\em Proceedings {CP 2004}}, pages 716--720, 2004.

\bibitem{bessiere-gac-schema}
C.~Bessi\`ere and J.-C. R\'egin.
\newblock Arc consistency for general constraint networks: preliminary results.
\newblock In {\em Proceedings 15th International Joint Conference on Artificial
  Intelligence ({IJCAI} 97)}, pages 398--404, 1997.

\bibitem{bessiere-regin-ac2001}
C.~Bessi\`ere and J.-C. R\'egin.
\newblock Refining the basic constraint propagation algorithm.
\newblock In {\em Proceedings 17th International Joint Conference on Artificial
  Intelligence ({IJCAI} 2001)}, pages 309--315, 2001.

\bibitem{caldwell_gent_underwood}
J.~Caldwell, I.~Gent, and J.~Underwood.
\newblock Search algorithms in type theory.
\newblock {\em Theoretical Computer Science}, 232(1-2):55--90, Feb. 2000.

\bibitem{Constable_naive}
R.~L. Constable.
\newblock Naive computational type theory.
\newblock In H.~Schwichtenberg and R.~Steinbruggen, editors, {\em Proof and
  System Reliability}, volume~62 of {\em Nato Science Series}, pages 213--259.
  Kluwer, 2002.

\bibitem{Nuprl}
R.~L.~C. et. al.
\newblock {\em Implementing Mathematics with the {N}uprl Proof Development
  System}.
\newblock Prentice Hall, 1986.

\bibitem{FreuderHandbook}
E.~C. Freuder and A.~K. Mackworth.
\newblock Constraint satisfaction: An emerging paradigm.
\newblock In P.~V.~B. F.~Rossi and T.~Walsh, editors, {\em Handbook of
  Constraint Programming}, pages 13--28. Elsevier, 2006.

\bibitem{gent-minion-2006}
I.~P. Gent, C.~Jefferson, and I.~Miguel.
\newblock Minion: A fast, scalable, constraint solver.
\newblock In {\em Proceedings 17th European Conference on Artificial
  Intelligence ({ECAI} 2006)}, pages 98--102, 2006.

\bibitem{Gent_Jefferson_Miguel06}
I.~P. Gent, C.~Jefferson, and I.~Miguel.
\newblock Watched literals for constraint propagation in minion.
\newblock {\em Constraint Programming 2006}, LNCS 4204:182--197, 2006.

\bibitem{nightingale_all_diff}
I.~P. Gent, I.~Miguel, and P.~Nightingale.
\newblock Generalised arc consistency for the alldifferent constraint: An
  empirical survey.
\newblock {\em Artificial Intelligence}, 172(18):1973--2000, 2008.

\bibitem{Girard}
J.-Y. Girard, Y.~Lafont, and P.~Taylor.
\newblock {\em Proofs and Types}.
\newblock Cambridge University Press, 1989.

\bibitem{VanHentenryckSimonisEa:92}
P.~V. Hentenryck, H.~Simonis, and M.~Dincbas.
\newblock Constraint satisfaction using constraint logic programming.
\newblock {\em Artificial Intelligence}, 58:113--159, 1992.

\bibitem{ilogsolver}
IBM.
\newblock {\em IBM ILOG CPLEX Optimization Studio CP Optimizer User’s Manual
  Version 12 Release 6}, 2015.

\bibitem{jeff-petrie-15}
C.~Jefferson and K.~E. Petrie.
\newblock Provably pointless propagator calls.
\newblock In {\em Fifth International Workshop on the Cross-Fertilization
  Between CSP and SAT}, 2015.
\newblock Co-located with {CP 2015}.

\bibitem{maher-prop-completeness}
M.~J. Maher.
\newblock Propagation completeness of reactive constraints.
\newblock In {\em Proceedings {ICLP 2002}}, pages 148--162, 2002.

\bibitem{chocosolver}
C.~Prud'homme, J.-G. Fages, and X.~Lorca.
\newblock {\em Choco3 Documentation}.
\newblock TASC, INRIA Rennes, LINA CNRS UMR 6241, COSLING S.A.S., 2014.

\bibitem{regin-gcc-96}
J.-C. R\'egin.
\newblock Generalized arc consistency for global cardinality constraint.
\newblock In {\em Proc. {AAAI} 96}, pages 209--215, 1996.

\bibitem{schultestuckey-propengines-08}
C.~Schulte and P.~J. Stuckey.
\newblock Efficient constraint propagation engines.
\newblock {\em ACM Transactions on Programming Languages and Systems (TOPLAS)},
  31(1), 2008.

\bibitem{csplib-carseq}
B.~Smith.
\newblock {CSPLib} problem 001: Car sequencing.
\newblock \url{http://www.csplib.org/Problems/prob001}.

\bibitem{Tack_Schulte_Smolka06}
G.~Tack, C.~Schulte, and G.~Smolka.
\newblock Generating propagators for finite set constraints.
\newblock {\em Constraint Programming 2006}, LNCS 4204:575--589, 2006.

\bibitem{csplib-roundrobin}
T.~Walsh.
\newblock {CSPLib} problem 026: Sports tournament scheduling.
\newblock \url{http://www.csplib.org/Problems/prob026}.

\end{thebibliography}

\end{document}